\documentclass[15pt,onecolumn]{IEEEtran}

\usepackage{amsfonts,amsmath,amssymb,amsthm}
\usepackage{tikz}
\usepackage{algorithm,algcompatible}
\newtheorem{theorem}{\bf Theorem}%[section]
\newtheorem{lemma}{\bf Lemma}
\newtheorem{corollary}{\bf Corollary}
\newtheorem{definition}{\bf Definition}
\newtheorem{example}{\bf Example}
\newtheorem{remark}{\bf Remark}

\newcommand{\mc}[1]{\mathcal{#1}}
\newcommand{\mbb}[1]{\mathbb{#1}}

\newcommand{\vect}[1]{\boldsymbol{#1}}
\newcommand{\set}[1]{\mathcal{#1}}

\newcommand{\ie}{{i.e.}}
\newcommand{\eg}{{e.g.}}
\newcommand{\etal}{et al.}

\title{Low-Complexity\\ Stochastic Generalized Belief Propagation}
\author{Farzin Haddadpour, ̃Mahdi Jafari Siavoshani, and Morteza Noshad%
\thanks{F.~Haddadpour is with Electrical Engineering Department, Shrif University of Technology, Tehran, Iran (e-mail: farzin\_haddadpour@alum.sharif.edu). ̃M.~Jafari~Siavoshani is with Computer Engineering Department, Shrif University of Technology, Tehran, Iran (e-mail: mjafari@sharif.edu). M. Noshad is with Electrical Engineering and Computer Science Department, University of Michigan, Ann Arbor, USA (e-mail: noshad@umich.edu).}%
}

\begin{document}
\maketitle

%\begin{center}
%\textbf{\Large }
%\end{center}
\begin{abstract}
The generalized belief propagation (GBP), introduced by Yedidia \etal, is an extension of the belief propagation (BP) algorithm, which is widely used in different problems involved in calculating exact or approximate marginals of probability distributions. 
In many problems, it has been observed that the accuracy of GBP considerably outperforms that of BP. However, because in general the computational complexity of GBP is higher than BP, its application is limited in practice.

%While, in general the complexity of GBP is higher than BP, in many problems its precision  considerably outperforms the BP algorithm. Therefore, if in an application the precision is of first priority, the GBP must be preferred over BP algorithm.

In this paper, we introduce a stochastic version of GBP called \emph{stochastic generalized belief propagation} (SGBP) that can be considered as an extension to the stochastic BP (SBP) algorithm introduced by Noorshams et al.
They have shown that SBP reduces the complexity per iteration of BP by an order of magnitude in alphabet size. 
In contrast to SBP,  SGBP can reduce the computation complexity if certain topological conditions are met by the region graph associated to a graphical model. However, this reduction can be larger than only one order of magnitude in alphabet size. In this paper, we characterize these conditions and the amount of computation gain that we can obtain by using SGBP.
%it should be emphasized that our proposed algorithm, meeting certain topological condition, is able to reduce the complexity per iteration of the parent-to-child algorithm of order ${\mc{O}}(d^{I})$ where $I$ is a quantity, depends on graph topology as well as region choices, which can be more than $1$. 
Finally, using similar proof techniques employed by Noorshams \etal, for general graphical models satisfy contraction conditions, we prove the asymptotic convergence of SGBP to the unique GBP fixed point, as well as providing non-asymptotic upper bounds on the mean square error and on the high probability error. 
\end{abstract}

\section{Introduction}
%Many researchers have been attracted by the topic of 
Graphical models and corresponding message-passing algorithms have attracted a great amount of attention due to their wide-spreading application in many fields, including signal processing, machine learning, channel and source coding, computer vision, decision making, and game theory (\eg, see \cite{kschischang2001factor, loeliger2004introduction}).
%. Further applications is mentioned in \cite{kschischang2001factor} and \cite{loeliger2004introduction}. 

Finding marginal and mode of a probability distribution are two basic problems encountered in the field of graphical models. Taking the rudimentary approach, the marginalization problem has exponentially growing complexity in alphabet size. However, using BP algorithm (firstly introduced in \cite{pearl1988probabilistic}) to solve this problem either exactly or approximately, we can reduce the computational complexity to a significant degree. It has been proved that applying BP on graphical models without cycles provides exact solution to the marginalization problem.
Furthermore, it has been observed that for general graphs, BP can find good approximations for marginalization (or finding mode) problems, \cite{kschischang2001factor, loeliger2004introduction}.

%\textcolor{red}{ "Marginalization problems"}, calculating the marginal distribution of a pmf in other words, as the essential part of the problems involved "Graphical Models" remain the very challenge ahead of us.  Taking the primitive approach, the marginalization problem has exponentially growing complexity. However, taking the advantage of the sum-product algorithm (which is also known as Belief Propagation (BP) algorithm and introduced firstly in \cite{pearl1988probabilistic}) to solve this problem either exactly or approximately we can reduce the computation complexity significantly. 
%In this algorithm, nodes in graph perform local computation and transmit results to their neighbors, which makes possible the calculation of marginals based on the received information when we end the running of the algorithm. 

%Though having some degrees of bias associated with sum-product algorithm as other mean field approaches in its rough calculations, different from other sampling methods; there is no deficiency such as high rate of variance to the BP algorithm which makes it more effective.  

Although BP has many favourable properties, it suffers from some limiting drawbacks. First, in complex and densely interconnected graphs, BP may not be able to produce accurate results; and even worse, it may not converge at all. Second, since in many applications (\eg, decoding of error-correcting codes) messages are of high dimensions, the computational complexity of BP algorithm will highly increase which leads to slow convergence rates.

To deal with the first drawback, some works have been done to propose alternative algorithms (\eg, see \cite{yuille2002cccp, welling2001belief, heskes2002stable, pakzad_estimation_2005}).
Specifically, to improve the accuracy of estimated marginal distribution, a generalization algorithm to BP has been introduced by Yedidia \etal~\cite{yedidia_constructing_2005}, known as Generalized Belief Propagation (GBP) algorithm. In their proposed algorithm, local computation is performed by a group of nodes instead of a single node as in BP. According to many empirical observations, GBP outperforms BP in many situations;  \cite{yedidia2001characterization, harel2003poset, welling2004choice, sibel2014application}. 
However, although GBP algorithm provides accurate results in terms of marginal distribution, it suffers from high order of computation complexity, specially in case of large alphabet size.

%Motivated by the work of Nima and Wainright, we propose a stochastic algorithm which takes the advantage of reducing the complexity along increasing the precision of the prediction simultaneously.

To overcome the second aforementioned deficiency of BP, lots of research have been conducted to reduce BP complexity for different applications (\eg, refer to \cite{felzenszwalb2006efficient,sudderth2010nonparametric,mcauley2011faster, isard2009continuously, kersting2009counting, coughlan2007dynamic, arulampalam2002tutorial, smith2013sequential}). 
%To mention a few more, \textcolor{red}{we can refer to  \cite{kschischang2001factor} which is related to the complexity reduction of Fourier transform to the linear factor degree, \cite{kersting2009counting} about speeding up the computation using the symmetry in factors. Furthermore, some quantization based method for bp update is used in \cite{coughlan2007dynamic}, \cite{isard2009continuously} and the works in \cite{arulampalam2002tutorial}, \cite{sudderth2010nonparametric} and \cite{smith2013sequential} relies on stochastic approach by particle filtering method}. 
In a recent work by Noorshams~\etal~\cite{noorshams_stochastic_2013}, to tackle with the challenge of high complexity in the case of large alphabet size, they introduce an alternative stochastic version of BP algorithm with lower complexity. The main idea behind their work is that each node sends a randomly sampled message taken from a properly chosen probability distribution instead of computing the exact message update rule in each iteration.

%However, we should highlight that in GBP algorithms, though having generally accurate result in terms of marginal distribution, these algorithms suffer from the high order of computation complexity. The computation complexity even gets exacerbated in the case of having large size alphabets. \\

Motivated by \cite{noorshams_stochastic_2013} and in order to mitigate the computational complexity of GBP, we extend GBP and propose stochastic GBP (SGBP) algorithm. SGBP  has the advantage of reducing the complexity, while increasing the accuracy of estimation.
In contrast to SBP, SGBP algorithm can reduce the computational complexity only if certain topological conditions are met by the region graph (defined later) associated to a graphical model. However, the complexity gain can be larger than only one order of magnitude in alphabet size. In this work, we characterize these conditions and the amount of computational gain that we can obtain by performing SGBP instead of GBP. Determining these criteria, we hope that they provide some useful guidelines on how to choose the regions and construct the region graph in a way that results to a lower complexity algorithm with good accuracy.
   
The rest of the paper is organized as follows.
First, \S\ref{sec:ProblemStatement} introduces our problem statement. In \S\ref{sec:SGBP_Alg}, we present the proposed stochastic GBP and then derive the topological conditions that guarantee SGBP has lower complexity than GBP.
Moreover, theoretical convergence results have been provided as well. Finally, to validate our theoretical results, considering a specific graphical model, SGBP is simulated and the results are presented. 
%For an extended version of this work please refer to \cite{haddadpour2016reduced}.

%--------------------------------------------------------------------
\section{Problem Statement}
\label{sec:ProblemStatement}
%--------------------------------------------------------------------
\subsection{Notation}
In the following, we introduce the notation that will be used in the paper. The random variables are represented by upper case letters and their values by lower case letters. Vectors and matrices are determined by bold letters. Sometimes, we use calligraphic letters to denote sets. When we have a set of random variables $X_1,\ldots,X_n$, we write $\vect{X}_\set{A}$ to denote $(X_i,i\in\set{A})$. An undirected graph $G = (\set{V},\set{E})$ is defined by a set of nodes $\set{V}=\{1,2,\ldots,n\}$ and a set of edges $\set{E} \subseteq \set{V}\times \set{V}$, where $(u, v) \in \set{E}$ if and only if nodes $u$ and $v$ are connected. Similarly, we can define a directed graph.

For every function $f(x_1,x_2,\ldots,x_n)$ where $f: {\mc{X}}^n \mapsto \mbb{R}$, we define the operator $\cal{L}$ as a map that turns this function to a vector ${\cal{L}}(f)\in\mbb{R}_{ {|\mc{X}|}^n\times 1}$ by evaluating $f$ at every input point.
For instance, considering ${\cal{X}}_{\{1,2\}}\in\{0,1\}$, for $f(x_1,x_2)$ we have
\[
{\cal{L}}(f) = \begin{bmatrix}
f(0,0)\\
f(0,1)\\
f(1,0)\\
f(1,1)\\
\end{bmatrix}.
\]

\subsection{Graphical Model}
Undirected graphical models, also known as Markov random fields (MRF), is a way to represent the probabilistic dependencies among a set of random variables having Markov properties using an undirected graph.
More precisely, we say that a set of random variables $X_1,\ldots,X_n$ form an MRF if there exists a graph $G=(\set{V},\set{E})$, where each $X_i$ is associated to the node $i \in \set{V}=\{1,\ldots,n\}$, and edges of the graph $G$ encode Markov properties of the random vector $\vect{X}=(X_1,\ldots,X_n)$. These Markov properties are equivalent to a factorization of the joint distribution of random vector $\vect{X}$ over the cliques of graph $G$ \cite{grimmett2010probability}.
In this paper, we focus on discrete random variables case where for all $j\in \set{V}$ we have $X_j\in{\mc{X}}\triangleq\{1,2,\ldots,d\}$. Moreover, we assume that the distribution of $\vect{X}$ is factorized according to
\[
p(\vect{x}) = \frac{1}{Z}\prod_{a\in \set{F}} \phi_a(\vect{x}_a)
\]
where $\set{F}$ is a collection of subsets of $\set{V}$ and $Z$ is a constant called the partition function. For the factor functions $\phi_a$, we have also $\phi_a\ge 0$. This factorization can be represented by using a bipartite graph $G_f=(\set{V},\set{F},\set{E}_f)$ called factor graph. In this representation, the variable nodes $\set{V}$ correspond to random variables $X_i$'s and factor nodes $\set{F}$ determine the factor functions $\phi_a$'s. Moreover, there exists an edge $(i,a)\in\set{E}_f$ between a variable node $i$ and a factor node $a$ if the variable $x_i$ appears in the factor $\phi_a$ (for more information on factor graphs refer to \cite{kschischang2001factor}).

\subsection{Region Graph}
In order to present the Yedidia's parent-to-child algorithm \cite{yedidia_constructing_2005} as well as introducing our stochastic GBP algorithm, we need to state some definitions as follows.

\begin{definition}[see \cite{yedidia_constructing_2005}]
A \emph{region graph} $G_r=(\set{R},\set{E}_r)$ defined over a factor graph $G_f=(\set{V},\set{F},\set{E}_f)$ is a directed graph in which for each vertex $v\in\set{R}$ (corresponding to a region) we have $v\subseteq \set{V}\cup\set{F}$. Each region $v$ has this property that if a factor node $a\in\set{F}$ belongs to $v$ then all of its neighbouring variable nodes have to also belong to $v$.
A directed edge $(v_p \rightarrow v_c)\in\set{E}_r$ 
%pointing from vertex $v_p$ to vertex $v_c$ 
may exist if $v_c\subset v_p$. If such an edge exists, $v_p$ is a \emph{parent} of $v_c$, or equivalently, $v_c$ is a \emph{child} of $v_p$.
If there exists a directed path from $v_a$ to $v_d$ on $G_r$, we say that $v_a$ is an \emph{ancestor} of $v_d$ and $v_d$ is a \emph{descendant} of $v_a$.
\end{definition}

Now, for each $R\in \set{R}$, we let $\set{P}(R)$ denotes for the set of all parents of  $R$, $\set{A}(R)$ denotes for the set of all ancestors of $R$ and $\set{D}(R)$ denotes for the set of all descendants of $R$.
Moreover, we define $\set{E}(R) \triangleq R \cup \set{D}(R)$. Finally, for a region $R\in\set{R}$, we use $|R|$ to denote for the number of variable nodes in $R$.

%\begin{remark}
%As mentioned in \cite{yedidia_constructing_2005}, having labeled nodes along allowing to have an edge between two regions with path of length two or more, differentiates the Region Graph from it is closely related Hasse Diagram.
%\end{remark}
%
%\begin{definition}
%Accumulating every edge in region graph we have:
%$${{\mathbf{E}}}=\{(R\longrightarrow S)|\:S\in{\cal{C}}(R)\}$$
%\end{definition}

\subsection{Parent-to-child GBP algorithm}
We may derive the BP message-passing equations using the fact that the belief at each variable node is the product of all the incoming messages received from its neighbouring factor nodes. Additionally, the beliefs over the set of variable nodes connecting to a factor node $a\in\set{F}$ is the product of the factor function $\phi_a$ multiplied by the incoming messages to the factor node $a$. Now marginalizing the second set of beliefs to find the belief over a variable node and equate it to the belief of that variable node which is found directly using the first equation, we can recover the BP update rules.

\begin{figure}[h]
   \centering
   \includegraphics[scale=	0.4]{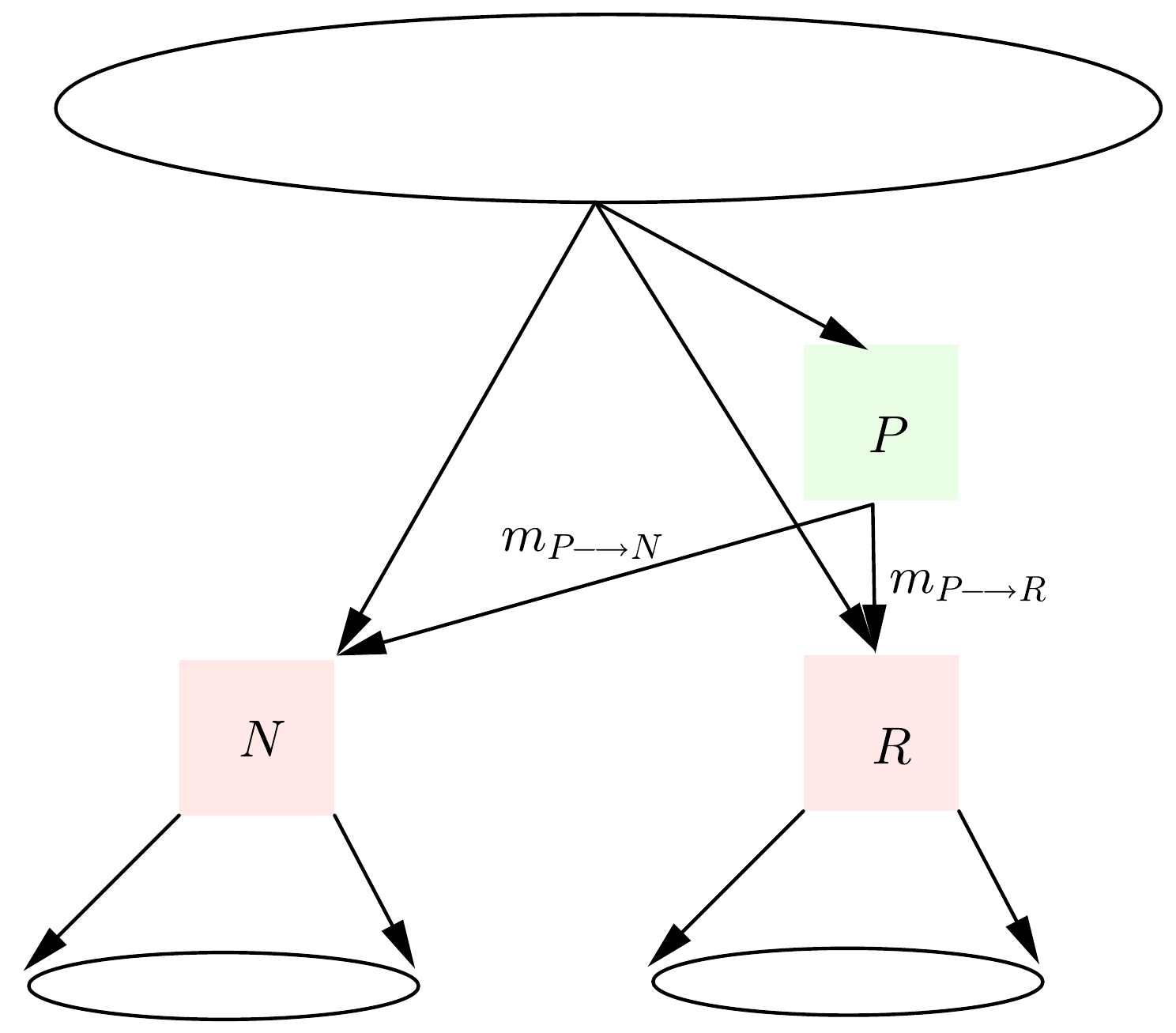}
   
   \caption{Graph region of an arbitrary graph corresponding to the parents-to-child algorithm.}
   \label{fig:ProblemSetup11}
   \end{figure}

Yedidia \etal, generalize this idea in \cite{yedidia_constructing_2005}, proposing an algorithm called parents-to-child GBP algorithm. As explained in \cite{yedidia_constructing_2005}, in the parent-to-child algorithm, we have only one kind of message $m_{P\rightarrow R}(\vect{x}_R)$ from a parent region $P$ to a child region $R$. Then for the belief of region $R\in\set{R}$ we have 
\begin{align}\label{eq:GBP_RegionBeliefComputation}
b_R(\vect{x}_R) \propto \Phi_{R}(\vect{x}_R) \times \prod_{P\in {\cal{P}}(R)}m_{P\rightarrow R}(\vect{x}_{R})
\times  \prod_{D\in {\cal{D}}(R)}\prod_{P'\in{\cal{P}}(D)\setminus \cal{E}(R)}m_{P'\rightarrow D}(\vect{x}_D) 
\end{align}
where $\Phi_{R}(\vect{x}_R) \triangleq \prod_{a\in R} \phi_a(\vect{x}_a)$ (with an abuse of notation when we product over $a\in R$ we mean to product only over the factor indexes of $R$).
Then, the message update rule over each edge $(P, R)\in\set{E}_r$ follows by
\begin{align}
m_{P\rightarrow R}(\vect{x}_{R}) =& \frac{\sum_{\vect{x}_{P\setminus R}}\Phi_{P\setminus R}(\vect{x}_{P'})\prod_{(I,J)\in N(P,R)}m_{I\rightarrow J}(\vect{x}_{J})}{\prod_{(I,J)\in D(P,R)}m_{I\rightarrow J}(\vect{x}_{J})} \label{eq:UpdateRule_FirstEquation}\\ 
%=& \sum_{\vect{x}_{P\setminus R}}\Phi_{P\setminus R}(\vect{x}_{P'})\frac{\prod_{(I,J)\in N(P,R)}m_{I\rightarrow J}(\vect{x}_{J})}{\prod_{(I,J)\in D(P,R)}m_{I\rightarrow J}(\vect{x}_{J})}\nonumber\\
=& \sum_{\vect{x}_{P\setminus R}}\Phi_{P\setminus R}(\vect{x}_{P'})\hat{M}(\vect{x}_{T_{PR}})\label{eq:TGBP}
\end{align}
where $\Phi_{P\setminus R}(\vect{x}_{P'})\triangleq \frac{\Phi_{P}}{\Phi_{R}}(\vect{x}_{P'})$ and ${P'}$ is the set of all variables appear in $\frac{\Phi_{P}}{\Phi_{R}}(\vect{x}_{P'})$. In addition, we have also
\begin{equation}
N(P,R) \triangleq \Big\{(I,J) | (I, J)\in \set{E}_r, I\notin \set{E}(P), J\in \set{E}(P) \setminus \set{E}(R) \Big\} \label{eq:Def_Set_N(P,R)}
\end{equation}
and
\begin{equation}
D(P,R) \triangleq \Big\{(I,J) | (I, J)\in \set{E}_r, I\in \set{D}(P) \setminus \set{E}(R), J\in \set{E}(R)  \Big\}. \label{eq:Def_Set_D(P,R)}
\end{equation}
%$N(P,R)$ is the set of all connected pairs of regions $(I,J)$ such that $J\in{\cal{E}}(P)$ but not $J\notin{\cal{E}}(R)$ while $I\notin{\cal{E}}(P)$. $D(P,R)$ is the set of all connected pairs of regions $(I,J)$ such that $J\in\cal{E}(R)$, while $I\in D(P)$, but $I\notin{\cal{E}}(R)$. 
Notice that the sets $N(P,R)$ and $D(P,R)$ can be calculated in advance. Moreover, $\hat{M}(\vect{x}_{T_{PR}})$ in \eqref{eq:TGBP} is defined as follows
\[
\hat{M}(\vect{x}_{T_{PR}})\triangleq\frac{\prod_{(I,J)\in N(P,R)}m_{I\rightarrow J}(\vect{x}_{J})}{\prod_{(I,J)\in D(P,R)}m_{I\rightarrow J}(\vect{x}_{J})},
\] 
where ${T_{PR}}$ is the set of all variables that appear in the above ratio.
%$\frac{\prod_{(I,J)\in N(P,R)}m_{I\rightarrow J}(\vect{x}_{J})}{\prod_{(I,J)\in D(P,R)}m_{I\rightarrow J}(\vect{x}_{J})}$. 

\begin{remark}
It can be easily observed that depending on the graph topology and the choice of regions, we may have either $P'\subset P$ or $P'=P$ in \eqref{eq:TGBP}.
For example, consider two pairwise Markov Random Fields presented in Figures \ref{fig:ProblemSetup-smallerp} and \ref{fig:biggerp}. 
Considering $P=\{1,2,4,5,7,8\}$ and $R=\{2,5,8\}$ in Figure~\ref{fig:ProblemSetup-smallerp}, we have $\Phi_{(124578\setminus 258)}(\vect{x}_{P'}) = \phi_1\phi_4\phi_7\psi_{12}\psi_{14}\psi_{74}\psi_{78}$
which leads to $P'=\{1,2,4,7,8\} \subset P$. 
On the other hand, choosing $P=\{1,2,4,5\}$ and $R=\{2,5\}$ in Figure~\ref{fig:biggerp}, we have $\Phi_{(1245\setminus 25)}(\vect{x}_{P'}) = \phi_1\phi_4\psi_{12}\psi_{14}\psi_{45}$. Hence, $P'=\{1,2,4,5\}=P$.
\hfill {\small $\blacksquare$}
\end{remark}

\begin{figure}[h]
    \begin{center}
    \begin{tikzpicture}[scale=1.8,very thick]
    \node at (-1.7,5.5)(index1){};
    \node at (-1.7,5) [circle,draw=black!50,fill=red!20!white,inner sep=8pt, minimum size=8mm] (node1l) {1};
   \node at (-1.7,4) [circle,draw=black!50,fill=red!20!white,inner sep=8pt, minimum size=8mm] (node2l) {4};
    \node at (-1.7,3) [circle,draw=black!50,fill=red!20!white,inner sep=8pt, minimum size=8mm] (node3l) {7};
  \node at (-0.85,4)[circle,dashed,draw=green!200,minimum size=65mm] (node3x){};
%    %\node at (-0.85,4.5)[circle,dashed,draw=black!50,minimum size=55mm] (node3x){};
%   % \node at (+0.85,3.5)[circle,dashed,draw=black!50,minimum size=55mm] (node3x){};
   \node at (+0.85,4)[circle,dashed,draw=green!200,minimum size=65mm] (node3x){};
%\node at (+0.85,4.5)[circle,dashed,draw=black!50,minimum size=55mm] (node3x){}; 
   \node at (0,5.5)(index2){};
    \node at (0,5) [circle,draw=black!50,fill=red!20!white,inner sep=8pt, minimum size=8mm] (node1) {2};
  \node at (0,4) [circle,draw=black!50,fill=red!20!white,inner sep=8pt, minimum size=8mm] (node2) {5};
    \node at (0,3) [circle,draw=black!50,fill=red!20!white,inner sep=8pt, minimum size=8mm] (node3) {8};
  
    \node at (1.7,5.5)(index3){};
    \node at (1.7,5) [circle,draw=black!50,fill=red!20!white,inner sep=8pt, minimum size=8mm] (node1r) {3};
    \node at (1.7,4) [circle,draw=black!50,fill=red!20!white,inner sep=8pt, minimum size=8mm] (node2r) {6};
    \node at (1.7,3) [circle,draw=black!50,fill=red!20!white,inner sep=8pt, minimum size=8mm] (node3r) {9};
  \draw[thick]  (node3l)--(node3)--(node3r)--(node2r)--(node1r)--(node1)--(node1l)--(node2l);\draw[thick] (node2l)--(node3l);\draw[thick](node1)--(node2)--(node3);
    
%    %\draw[->,dashed,blue] (node1)--(node6r);
    \end{tikzpicture}
    \caption{Graph region of a graph with $x_{P'}\subset x_{P}$.}
    \label{fig:ProblemSetup-smallerp}
    \end{center}
\end{figure}
%\begin{figure}[h]
  % \centering
  % \includegraphics[scale=	0.7]{Fig8}
  % 
   %\end{figure}

  \begin{figure}[h]
    \begin{center}
   \begin{tikzpicture}[scale=1.2,very thick]
     \node at (-2.5,5) [rectangle,draw=black!50,fill=green!20!white,inner sep=10pt, minimum size=10mm] (node1l) {124578};
 \node at (0,3) [rectangle,draw=black!50,fill=red!20!white,inner sep=10pt, minimum size=10mm] (node) {258};
    \node at (2.5,5) [rectangle,draw=black!50,fill=green!20!white,inner sep=10pt, minimum size=10mm] (node1r) {235689};
  \draw[thick,->](node1l)--(node);\draw[thick,->](node1r)--(node);
    \end{tikzpicture}
   \caption{Graph region of a graph with $x_{P'}\subset x_{P}$.}
      \label{fig:ProblemSetup-smallerreg}
    \end{center}
    \end{figure}
   \begin{figure}[h]
   \begin{center}
    \begin{tikzpicture}[scale=2,very thick]
    \node at (-1.7,5.5)(index1){};
    \node at (-1.7,5) [circle,draw=black!50,fill=blue!20!white,inner sep=8pt, minimum size=8mm] (node1l) {1};
    \node at (-1.7,4) [circle,draw=black!50,fill=blue!20!white,inner sep=8pt, minimum size=8mm] (node2l) {4};
    \node at (-1.7,3) [circle,draw=black!50,fill=blue!20!white,inner sep=8pt, minimum size=8mm] (node3l) {7};
    \node at (-0.85,3.5)[circle,dashed,draw=black!50,minimum size=55mm] (node3x){};
    \node at (-0.85,4.5)[circle,dashed,draw=black!50,minimum size=55mm] (node3x){};
    \node at (+0.85,3.5)[circle,dashed,draw=black!50,minimum size=55mm] (node3x){};
     \node at (+0.85,4.5)[circle,dashed,draw=black!50,minimum size=55mm] (node3x){};
     \node at (+0.85,4.5)[circle,dashed,draw=black!50,minimum size=55mm] (node3x){}; 
    \node at (0,5.5)(index2){};
    \node at (0,5) [circle,draw=black!50,fill=blue!20!white,inner sep=8pt, minimum size=8mm] (node1) {2};
    \node at (0,4) [circle,draw=black!50,fill=blue!20!white,inner sep=8pt, minimum size=8mm] (node2) {5};
    \node at (0,3) [circle,draw=black!50,fill=blue!20!white,inner sep=8pt, minimum size=8mm] (node3) {8};
    
    \node at (1.7,5.5)(index3){};
    \node at (1.7,5) [circle,draw=black!50,fill=blue!20!white,inner sep=8pt, minimum size=8mm] (node1r) {3};
   \node at (1.7,4) [circle,draw=black!50,fill=blue!20!white,inner sep=8pt, minimum size=8mm] (node2r) {6};
    \node at (1.7,3) [circle,draw=black!50,fill=blue!20!white,inner sep=8pt, minimum size=8mm] (node3r) {9};
 \draw[thick]  (node3l)--(node3)--(node3r)--(node2r)--(node1r)--(node1)--(node1l)--(node2l)--(node2)--(node2r);\draw[thick] (node2l)--(node3l);\draw[thick](node1)--(node2)--(node3);
    
%    %\draw[->,dashed,blue] (node1)--(node6r);
    \end{tikzpicture}
    \caption{Basic clusters in 9 nodes grid with $x_{P'}=x_{P}$}\label{fig:biggerp}
    \end{center}
   \end{figure}
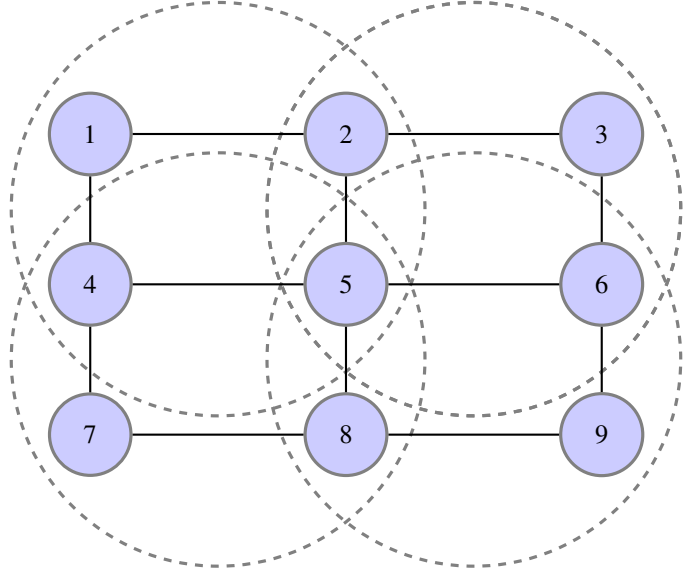

\begin{remark}
In the parent-to-child algorithm, the message transmitted over each edge $(P, R) \in \set{E}_r$ can be considered as a vector by applying the operator $\mc{L}(\cdot)$.
Namely, by concatenating all possible messages, we define $\vect{m}_{P\rightarrow R} \triangleq \mc{L}(m_{P\rightarrow R})$ where $\vect{m}_{P\rightarrow R} \in \mbb{R}^{d^{|R|}}$.
Moreover, concatenating all the messages over all edges of the region graph, we define $\vect{m} \triangleq \{\vect{m}_{P\rightarrow R}\}_{ (P, R)\in \set{E}_r} \in \mbb{R}^\Delta$ where $\Delta=\sum_{(P,R)\in \set{E}_r} d^{|R|}$. 
\hfill {\small $\blacksquare$}
\end{remark}

Now, we can state the complexity of the parent-to-child GBP algorithm as stated in Lemma~\ref{lem:GBP_Complexity}.
\begin{lemma}\label{lem:GBP_Complexity}
The computation complexity of the message update rule of the parent-to-child GBP algorithm associated with each edge, computed according to \eqref{eq:TGBP}, is ${\mc{O}}(d^{|P|})$. 
\end{lemma}
%--------
\begin{proof}
For each fixed vector $\vect{x}_R$, the calculation of $m_{P\rightarrow R}(\vect{x}_{R}) =\sum_{\vect{x}_{P\setminus R}}\Phi_{P\setminus R}(\vect{x}_{P'})\hat{M}(\vect{x}_{T_{PR}})$ 
needs $d^{|P\setminus R|}$ operations. Moreover, to find $m_{P\rightarrow R}(\cdot)$ completely, one needs to evaluate the above summation $\mc{O}(d^{|R|})$ times.
%Additionally, $m_{P\longrightarrow C}(\vect{x}_{C}) $ is a function of the vector $\vect{x}_{C}$ which demands $\mc{O}(d^|C|)$ different sampling operation. 
Consequently, the overall complexity of calculating $m_{P\longrightarrow R}(\vect{x}_{R}) $ is of the order ${\mc{O}}(d^{|R|}\times d^{|P\setminus R|})={\cal{O}}(d^{|P|})$.
\end{proof}

At each round of the parent-to-child algorithm, $t=1,2,\ldots$, every parent node $P$ of $R$ in the region graph calculates a message $m^{(t+1)}_{P\rightarrow R}$ and sends it to node $R$. Mathematically, this can be written as (see \cite{yedidia_constructing_2005})
\begin{align}\label{eq:UpdateRuleExpansion1}
m^{(t+1)}_{P\rightarrow R} (\vect{x}_R) =& \left[ \Upsilon_{P\rightarrow R}(m^{(t)}) \right] (\vect{x}_R) \nonumber\\
%
% & \hspace{-2cm} = \sum_{\vect{x}_{P\setminus R}}\Phi_{P\setminus R}(\vect{x}_{P'}) \frac{\prod_{(I,J)\in N(P,R)} m^{(t)}_{I\rightarrow J}(\vect{x}_{J})}{\prod_{(I,J)\in D(P,R)}m^{(t)}_{I\rightarrow J}(\vect{x}_{J})}\nonumber\\
%
& \hspace{-2cm} = \sum_{\vect{x}_{P\setminus R}}\Phi_{P\setminus R}(\vect{x}_{P'})\hat{M}^{(t)}(\vect{x}_{T_{PR}}) \nonumber\\
& \hspace{-2cm}  = \sum_{\vect{x}_{(P\setminus R) \setminus T_{PR} }}   \sum_{\vect{x}_{(P\setminus R)\cap T_{PR} }}  \Phi_{P\setminus R}(\vect{x}_{P'}) \hat{M}^{(t)}(\vect{x}_{ T_{PR}} )  \nonumber\\
& \hspace{-2cm} = k_{PR}^{(t)} \left( \vect{x}_{ T_{PR} \setminus  (P\setminus R) } \right) \hspace{-11pt} \sum_{\vect{x}_{(P\setminus R) \setminus  T_{PR} }} \sum_{\vect{x}_{(P\setminus R)\cap  T_{PR} }} \Big[ \Phi_{P\setminus R}(\vect{x}_{P'}) \times Q^{(t)}(\vect{x}_{ T_{PR} \cap (P\setminus R)} | \vect{x}_{ T_{PR} \setminus (P\setminus R) }) \Big],
%
%&\frac{\hat{M}^{(t)}(\vect{x}_{T_{PR}})}{\sum_{\vect{x}'_{ T_{PR} \cap (P\setminus R)}} \hat{M}^{(t)} \left( \vect{x}'_{ T_{PR} \cap (P\setminus R)} , \vect{x}_{ T_{PR} \setminus (P\setminus R) } \right) } \nonumber\\
\end{align}
where
\begin{align}\label{eq:Dist_Q_Def}
Q^{(t)}(\vect{x}_{ T_{PR} \cap (P\setminus R)} | \vect{x}_{ T_{PR} \setminus (P\setminus R) }) \triangleq\frac{\hat{M}^{(t)} \left( \vect{x}_{ T_{PR} \cap (P\setminus R) } , \vect{x}_{ T_{PR} \setminus (P\setminus R)} \right) }{\sum_{\vect{x}'_{ T_{PR} \cap (P\setminus R) }} \hat{M}^{(t)} \left( \vect{x}'_{ T_{PR} \cap (P\setminus R) }, \vect{x}_{ T_{PR} \setminus (P\setminus R) } \right) }
\end{align}
is a conditional distribution. Moreover, 
\[
k_{PR}^{(t)} \left( \vect{x}_{ T_{PR} \setminus  (P\setminus R) } \right) \triangleq  \hspace{-12pt}  \sum_{\vect{x}'_{ T_{PR} \cap (P\setminus R)}}  \hspace{-12pt}  \hat{M}^{(t)} \left( \vect{x}'_{ T_{PR} \cap (P\setminus R)} , \vect{x}_{ T_{PR} \setminus  (P\setminus R) } \right).
\]

Hence, for the update rule we can write
\begin{align}
m^{(t+1)}_{P\rightarrow R} (\vect{x}_R) = k_{PR}^{(t)} %\left( \vect{x}_{ T_{PR} \setminus (P\setminus R)  } \right) 
\hspace{-10pt}  \sum_{\vect{x}_{(P\setminus R) \setminus  T_{PR} } }  \hspace{-10pt}  \mbb{E}_{[\vect{X}_{(P\setminus R)\cap  T_{PR} } \sim Q^{(t)} ] } \big[\Phi_{P\setminus R}(\vect{X}_{P'})\big] \label{eq:UpdateRuleExpansion2}
%\sum_{P\setminus R\cap T}\hat{M}^{(t)}(\vect{x}_T | \vect{x}_{T\setminus\{P\setminus R\cap T\}}) \label{eq:UpdateRuleExpansion}
\end{align}
Here and in the following, for brevity and clarity of notation, we will omit the dependence of $k_{PR}^{(t)}$ to the variables  $\vect{x}_{ T_{PR} \setminus (P\setminus R) }$.

Now, notice that we can decompose the set $P'$ as follows
\[
P' = \left[ (P\setminus R)\cap T_{PR} \right] \cup \left[ (P\setminus R)\setminus T_{PR} \right] \cup \left[P' \setminus (P\setminus R) \right]
\]
because we always have $P\setminus R \subseteq P'$. By using this relation, we can rewrite \eqref{eq:UpdateRuleExpansion2} as follows
\begin{align} \label{eq:UpdateRuleExpression_Final}
m^{(t+1)}_{P\rightarrow R} (\vect{x}_R) = k_{PR}^{(t)} %\left( \vect{x}_{ T_{PR} \setminus (P\setminus R)  } \right)  
%\hspace{-11pt} 
\sum_{\vect{x}_{(P\setminus R) \setminus  T_{PR} } } 
%\hspace{-11pt}
\mbb{E}_{ [\vect{X}_{(P\setminus R)\cap T_{PR} } \sim Q^{(t)} ] } \bigg[
  \Phi_{P\setminus R} \Big(\vect{X}_{(P\setminus R)\cap T_{PR} }, \vect{x}_{(P\setminus R)\setminus T_{PR} }, \vect{x}_{P'\setminus (P\setminus R)} \Big) \bigg]. 
\end{align}

In \eqref{eq:UpdateRuleExpansion1}, $\Upsilon_{P\rightarrow R}:\mbb{R}^{\Delta} \mapsto \mbb{R}^{d^{|R|}}$ is the local update function of the directed edge $(P, R)\in \set{E}_r$. By concatenating all the local update functions over the edges of the region graph, we can define the global update function as 
\begin{equation}\label{eq:GlobalUpdateFuncDef}
\Upsilon(\vect{m})=\Big[ \Upsilon_{P\rightarrow R}(\vect{m}) :  (P,R)\in \set{E}_r \Big]
\end{equation}
where $\Upsilon: \mbb{R}^\Delta \mapsto \mbb{R}^\Delta$.
The goal of the (parent-to-child) GBP algorithm is to find a fixed point $\vect{m}^*$ that satisfies $\Upsilon(\vect{m}^*)=\vect{m}^*$. If a fixed point $\vect{m}^*$ is found, then the beliefs of random variables in a region $R$ is computed by applying \eqref{eq:GBP_RegionBeliefComputation}.

%where $\Upsilon_{RS}:\mbb{R}^{d^{m |R| }}\longrightarrow\mathbb{R}^{d^{|S|}}$ is the local update function of the directed edge $R\longrightarrow S$ and \textcolor{red}{$m(|R|)$ is to be defined...}. To our analysis purpose of proposed algorithm, we define the global update function $\Upsilon(m)=\{\Upsilon_{RS}(m)\}_{\{(R\longrightarrow S):S\subset R\}}$ by concatenating all such local update functions, then if the messages obtained from SBP updates converge to $m^*$, then we have the fixed point relationship $\Upsilon(m^*) = m^*$.  

%====================================================================
\section{Stochastic Generalized Belief Propagation Algorithm}
\label{sec:SGBP_Alg}
In this section, first we introduce our stochastic extension to the parent-to-child GBP algorithm, and then present a result on the criteria where this algorithm is able to mitigate the computation complexity of GBP. 
%Next, using sufficiently general example we write the parent-to-child algorithm then we introduce how to rebuild our algorithm based on these relations. Finally, we compare the complexity of the algorithms.

Based on \eqref{eq:UpdateRuleExpression_Final}, we introduce our algorithm as stated in Algorithm~\ref{alg:SGBP_Alg}. The main idea of the algorithm is that under proper conditions (that will state in Theorem~\ref{thm:mainresult}), some parts of the message update rule \eqref{eq:TGBP} for each edge of the region graph can be written as an expectation as stated in \eqref{eq:UpdateRuleExpression_Final}. 
%So as a result, Algorithm~\ref{alg:SGBP_Alg} does not perform the stochastic update rule for all edges but for those satisfy the conditions of Theorem~\ref{thm:mainresult}.

\begin{algorithm}
    \caption{Stochastic Generalized Belief Propagation (SGBP) algorithm.}
    %for a general graph region (see \eqref{eq:UpdateRuleExpression_Final} and Figure~\ref{fig:ProblemSetup11})}
    \label{alg:SGBP_Alg}
  \begin{algorithmic}[1]
    \STATE \textbf{Initialize the messages.} 
    \FOR{$t\in\{1,2,\ldots\}$ and each directed edge $(P, R)\in\set{E}_r$}
    	%\IF{all or at least the first and the third conditions of Theorem~\ref{thm:mainresult} hold}
      %\STATE  Use the ordinary parents-to-child GBP algorithm to those messages to be sent from the second to the next layers along those edges which do not satisfy the condition II 
        %\COMMENT{to be mentioned.}
      		\STATE Choose a random vector $\vect{J}_{PR}^{(t+1)} \in \mc{X}^{| T_{PR} \cap (P\setminus R)|}$ according to the conditional distribution $Q^{(t)}(\vect{x}_{ T_{PR} \cap (P\setminus R)} | \vect{x}_{ T_{PR} \setminus (P\setminus R) })$ defined in \eqref{eq:Dist_Q_Def}.
      		%= \frac{\hat{M} ^{(t)}\left( \vect{x}_{ T_{PR} \cap (P\setminus R) } , \vect{x}_{ T_{PR} \setminus (P\setminus R)} \right) }{\sum_{\vect{x}_{(P\setminus R)\cap T_{PR} }\hat{M}^{(t)} \left( \vect{x}_{ T_{PR} \cap (P\setminus R) }, \vect{x}_{ T_{PR} \setminus (P\setminus R) } \right) }}$
      		\STATE Update the message $m_{P\longrightarrow R}^{(t+1)}$ with the appropriately tuned step size  $\alpha^{(t)}={\mathcal{O}}(\frac{1}{t})$ according to  
      \begin{align}\label{eq:UpdateRule_Alg}
      m_{P\rightarrow R}^{(t+1)} (\vect{x}_R) =(1-\alpha^{(t)}) m_{P\rightarrow R}^{(t)} (\vect{x}_R) +\alpha^{(t)}  k_{PR}^{(t)} (\vect{x}_{T_{PR}\setminus (P\setminus R)})  \hspace{-10pt}  \sum_{\vect{x}_{(P\setminus R)\setminus T_{PR} }}  \hspace{-14pt}  \Phi_{P\setminus R} \left(\vect{J}_{PR}^{(t+1)}, \vect{x}_{(P\setminus R)\setminus T_{PR} }, \vect{x}_{P'\setminus (P\setminus R)} \right) 
      \end{align}
		%\ELSE      
      		%\STATE Run the ordinary parent-to-child GBP algorithm to find the message $m^{(t+1)}_{P\rightarrow R}$.
     % 	\ENDIF	 
      \STATE $t = t + 1$
    \ENDFOR
%    \STATE Update the marginal at every node of the region graph $R\in V$ by 
%     \[
%       b^{(t+1)}_R(\vect{x}_R) \propto \Phi_{R}(\vect{x}_R) \left( \prod_{P\in {\cal{P}}(R)}m^{(t+1)}_{P\longrightarrow R}(\vect{x}_{R}) \right)\cdot \left( \prod_{D\in {\cal{D}}(R)}\prod_{P'\in{\cal{P}}(D)\setminus \cal{E}(R)}m^{(t+1)}_{P'\longrightarrow D}(\vect{x}_D) \right) 
%      \]      
  \end{algorithmic}
\end{algorithm}

\begin{remark}\label{remrk:DeterministicUpdateRule}
Note that when $(P\setminus R)\cap T_{PR}=\varnothing$, the update rule \eqref{eq:UpdateRule_Alg} becomes deterministic as stated in the following
\[
m_{P\rightarrow R}^{(t+1)}=(1-\alpha^{(t)})m_{P\rightarrow R}^{(t)} +\alpha^{(t)}  k_{PR}^{(t)} (\vect{x}_{T_{PR}})  \Big[\sum_{\vect{x}_{(P\setminus R)}}\Phi_{P\setminus R} \left(\vect{x}_{(P\setminus R) }, \vect{x}_{P'\setminus (P\setminus R)} \right)\Big].
\]
Later, we will prove in Lemma~\ref{lem:Graph-region property} that this condition can only happen in update rules corresponding to the highest-level ancestors regions.
\hfill {\small $\blacksquare$}
\end{remark}

In contrast to SPB studied in \cite{noorshams_stochastic_2013}, the stochastic version of GBP does not always reduce the computational complexity in each iteration. Theorem~\ref{thm:mainresult} describes the topological and regional conditions for which the complexity of SGBP is less than GBP for a specific edge of the region graph.

\begin{theorem}\label{thm:mainresult}
Our proposed algorithm that runs over a region graph $G_r$ reduces the computation complexity of each message $m_{P\rightarrow R}$ (compared to GBP) if and only if the following conditions hold
%  Regarding Figure \ref{fig:ProblemSetup11}, our proposed algorithm, in the case of being convergent, reduces the complexity within each edge if stochastic GBP meets the following conditions: 
  \begin{itemize}
  \item[(i)] $(P\setminus R)\cap T_{PR} \neq\varnothing$
  \item[(ii)] $(P\setminus R) \nsubseteq T_{PR}$
%  \item[(iii)]$T_{PR} \neq\varnothing$
  \end{itemize}
\end{theorem}
%----------
\begin{proof}
The main idea of the proof lies in the fact that whether or not \eqref{eq:TGBP} can be written in the form of an expected value of \emph{potential functions} as stated in \eqref{eq:UpdateRuleExpression_Final}. If this happens, as presented in Algorithm~\ref{alg:SGBP_Alg}, the complexity of update rules can be reduced. Now, to be able to have an expectation operation in \eqref{eq:UpdateRuleExpression_Final}, we should have $(P\setminus R)\cap T_{PR} \neq\varnothing$.

Now, assuming condition (i) holds, we find the complexity of Algorithm~\ref{alg:SGBP_Alg}'s update rule over every edge $(P, R)\in\set{E}_r$ in each iteration. First, let us fix $\vect{x}_R$. To find the PMF of the random vector $\vect{J}$ which is given by \eqref{eq:Dist_Q_Def},
%\[
%\Pr(\vect{x}_{ T_{PR} \cap (P\setminus R)} | \vect{x}_{ T_{PR} \setminus (P\setminus R) }) = \frac{\hat{M}^{(t)} \left( \vect{x}_{ T_{PR} \cap (P\setminus R) } , \vect{x}_{ T_{PR} \setminus (P\setminus R)} \right) }{\sum_{\vect{x}'_{ T_{PR} \cap (P\setminus R) }\hat{M}^{(t)} \left( \vect{x}'_{ T_{PR} \cap (P\setminus R) }, \vect{x}_{ T_{PR} \setminus (P\setminus R) } \right) }}
%\]
we need ${\mc{O}}(d^{|\{P\setminus R\}\cap T_{PR} |}\times d^{| T_{PR} \setminus \{P\setminus R\}|}) = {\mc{O}}(d^{| T_{PR} |})$ operations. Notice that since we have $ T_{PR} \setminus (P\setminus R) \subseteq R$ and $[(P\setminus R)\cap T_{PR} ] \cap [(P\setminus R)\setminus T_{PR} ]=\varnothing$, for every fixed $\vect{x}_R$, the PMF of $\vect{J}$ does not depend on the vector $\vect{x}_{(P\setminus R)\setminus T_{PR} }$. 
This means that for a fixed $\vect{x}_R$, to find the summation in \eqref{eq:UpdateRule_Alg}, the PMF of $\vect{J}$ should only computed once.

Hence, the overall complexity of update rule \eqref{eq:UpdateRule_Alg} becomes
\[
\mc{O} \left( d^{| T_{PR} |}+d^{|R|}\big[d^{|(P\setminus R)\cap T_{PR} |} + d^{|(P\setminus R)\setminus T_{PR} |}+ d^{|(P\setminus R)\cap T_{PR} |}\big] \right) 
\]
where the terms in the brackets count for a fixed $\vect{x}_R$ the computation complexity of $k(\vect{x}_{ T_{PR} \setminus {P\setminus R}})$, of the summation in \eqref{eq:UpdateRule_Alg}, and of taking a sample vector $\vect{J}$ from the above PMF, respectively. The above relation can be rewritten as follows
\[
\mc{O} \left( \max\big[d^{| T_{PR} |}, d^{|R|+|(P\setminus R)\setminus T_{PR} |},d^{|R|+|(P\setminus R)\cap T_{PR} |}\big] \right)
\]

Now, we can conclude that if $T_{PR} \neq \varnothing$ and $(P\setminus R)\not\subset T_{PR}$ then we have
\[
\mc{O} \left( \max\big[d^{| T_{PR} |}, d^{|R|+|(P\setminus R)\setminus T_{PR} |},d^{|R|+|(P\setminus R)\cap T_{PR} |}\big] \right) < {\cal{O}}(d^{|P|}),
\]
where the right hand side is the computation complexity of the parent-to-child GBP algorithm derived in Lemma~\ref{lem:GBP_Complexity}. This completes the proof of theorem.
\end{proof}

%\begin{remark}
%It worth to mention that the linearity property of the ordinary BP and Parent-to-Child GBP algorithm allows us to devise SBP and SGBP with lower order complexity.
%\end{remark}

\begin{corollary}\label{cor:orderreduction}
Assuming that the conditions of Theorem~\ref{thm:mainresult} hold and denoting 
\[
\eta_{PR}\triangleq\max{\Big[ | T_{PR} |,|R|+|(P\setminus R)\setminus T_{PR} |, |R|+|(P\setminus R)\cap T_{PR} | \Big]},
\]
Algorithm~\ref{alg:SGBP_Alg} reduces the computation complexity of message $m_{P\rightarrow R}$ of the order 
$\mathcal{O}(d^{|P|-\eta_{PR}})=\mathcal{O}(d^{I_{PR}})$
where $I_{PR} \triangleq |P|-\eta_{PR}$. Notice that $I_{PR}$ can be larger than $1$.
\end{corollary}

\begin{example}\label{ex:RegionGraph_with_I=2}
In this example, we provide a graph, drawn in Figure~\ref{fig:I=2}, in which SGBP reduces the complexity of updating rule for the edge $123456\rightarrow 36$ (see Figure~\ref{fig:graph region of I=2}) with two order of magnitude. The updating rule for  $m_{123456\rightarrow36}$ is as follows
\begin{align*}
m_{123456\rightarrow36}=& \sum_{x_1x_2x_4x_5} \frac{\Phi_{123456}(x_1,x_2,x_3,x_4,x_5,x_6)}{\Phi_{36}(x_3,x_6)} \times \frac{m_{2478\rightarrow 24}(x_2,x_4)}{1}\\
=&\sum_{x_2x_4}\tilde{\Phi}(x_2,x_3,x_4,x_6) m_{2478\rightarrow 24}(x_2,x_4)\\
=& k\:\mathbb{E}_{(X_2,X_4)\sim Q} \left[ \tilde{\Phi}(X_2,x_3,X_4,x_6) \right]
\end{align*}
where in this example by applying SGBP we will get $I_{123456\rightarrow36}=2$.
%----------
\begin{figure}
\begin{center}
\includegraphics[scale=0.45]{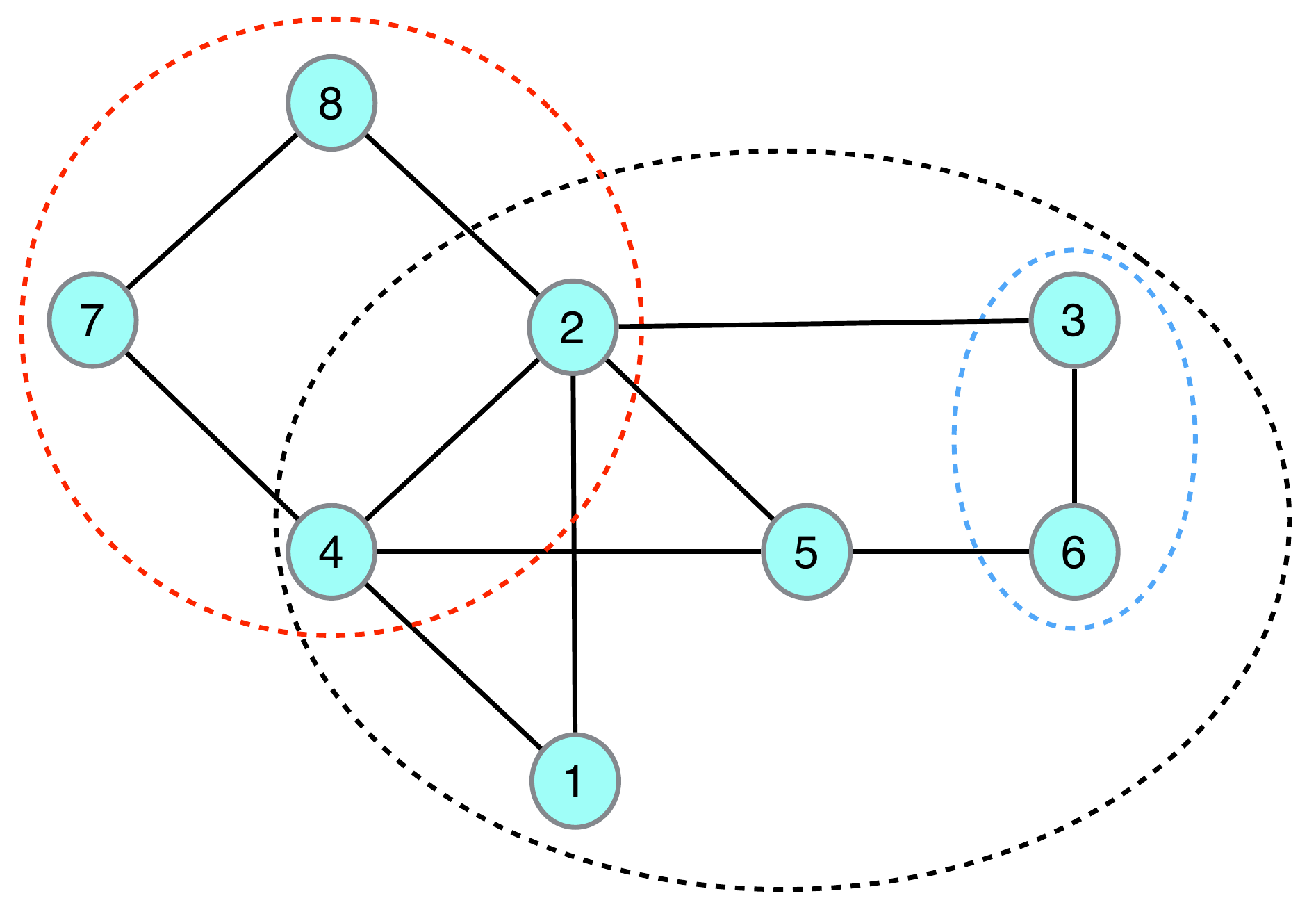}
\caption{A graph in which the SGBP reduces the complexity with two order of magnitude in alphabet size (see Example~\ref{ex:RegionGraph_with_I=2}).}
\label{fig:I=2}
\end{center}
\end{figure}
%----------
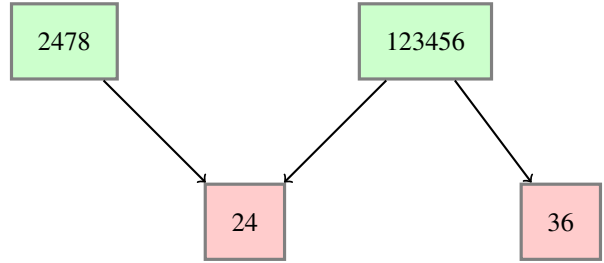
\begin{figure}[h]
    \begin{center}
   \begin{tikzpicture}[scale=1.2,very thick]
     \node at (-1.5,5) [rectangle,draw=black!50,fill=green!20!white,inner sep=10pt, minimum size=10mm] (node1l) {2478};
 \node at (4,3) [rectangle,draw=black!50,fill=red!20!white,inner sep=10pt, minimum size=10mm] (node1rr) {36};
  \node at (0.5,3) [rectangle,draw=black!50,fill=red!20!white,inner sep=10pt, minimum size=10mm] (node) {24};
    \node at (2.5,5) [rectangle,draw=black!50,fill=green!20!white,inner sep=10pt, minimum size=10mm] (node1r) {123456};
  \draw[thick,->](node1l)--(node);\draw[thick,->](node1r)--(node);
  \draw[thick,->](node1r)--(node1rr);
    \end{tikzpicture}
   \caption{Associated graph region of a graphical model depicted in Figure~\ref{fig:I=2} with $I_{123456\rightarrow36}=2$ (see Example~\ref{ex:RegionGraph_with_I=2}).}
      \label{fig:graph region of I=2}
    \end{center}
\end{figure}
%----------
\hfill {\small $\blacksquare$}
\end{example}

\begin{corollary}
The complexity of the parent-to-child GBP algorithm is dominated by the computation complexity of message update rule of the highest level edges in the region graph $G_r$.
As a result, if the dominant message update rule that belongs to the highest-level ancestor regions with the largest size, satisfies the conditions of Theorem~\ref{thm:mainresult}, then no matter what are the complexity of other edges, Algorithm~\ref{alg:SGBP_Alg} will reduce the overall computation complexity of the parent-to-child GBP. 
\end{corollary}
 
\subsection{Convergence Rate of SGBP Algorithm}
\label{sec:Convergence_SGBP}
In this section, we extend the convergence guarantees of \cite{noorshams_stochastic_2013} to SGBP. Our convergence theorem (Theorem~\ref{thm:Convergence_Results}) is based on imposing a sufficient condition similar to \cite{noorshams_stochastic_2013} that guarantees uniqueness and convergence of the parent-to-child GBP message updates. More precisely we assume that the global update function $\Upsilon(\cdot)$, defined in \eqref{eq:GlobalUpdateFuncDef}, is contractive, namely $\exists \nu, 0<\nu<2$ such that
%This assumption relies on the global update function, $\Upsilon(\vect{m})$ and assuming that $\exists \nu$ that $0<\nu< 2$, as following:
\begin{align}
\|\Upsilon(\vect{m})-\Upsilon(\vect{m}')\|_2 \leq  \left( 1-\frac{\nu}{2} \right) \| \vect{m}-\vect{m}' \|_2. \label{eq:ContractiveCond_Def}
\end{align} 

Following similar proof technique to \cite{noorshams_stochastic_2013}, with some appropriate modifications, we can obtain the following results.
%\begin{theorem}[General topology]\label{thm:sndmainresult}

\begin{theorem}\label{thm:Convergence_Results}
Assume that, for a given region graph, the update function $\Upsilon$ is contractive  with parameter $1-\frac{\nu}{2}$ as defined in \eqref{eq:ContractiveCond_Def}. Then, parent-to-child GBP has a unique fixed point $\vect{m}^*$ and the message sequence $\{\vect{m}^{(t)}_{P\rightarrow R}\}^{\infty}_{t=1}$ generated by the SGBP algorithm has the following properties: 
\begin{itemize}
\item[i)] The result of SGBP is consistent with GBP, namely we have $\vect{m}^{(t)}\overset{\mathrm{a.s.}}{\longrightarrow}\vect{m}^*$ as $t\longrightarrow \infty$.
\item[ii)] Bounds on mean-squared error: Let us divide the fixed point message $\vect{m}^*$ into two parts,  $\vect{m}^* = {(\vect{m}_{\mc{E}_{1}}^*,\vect{m}_{\mc{E}_{\sim 1}}^*)}$, where $\vect{m}^*_{\mc{E}_{r_1}}$ corresponds to those edges of the region graph that perform deterministic update rule (as stated in Algorithm~\ref{alg:SGBP_Alg} and Remark~\ref{remrk:DeterministicUpdateRule}), while $\vect{m}^*_{\mc{E}_{\sim 1}}$ corresponds to the  edges that run the stochastic algorithm. In other words,  $\mc{E}_{1}$ and $\mc{E}_{\sim 1}$ represents the edges in region graph in which the message updating rules are deterministic and stochastic respectively.
Choosing step size $\alpha^{(t)}=\frac{\alpha}{\nu(t+2)}$ for some fixed $1<\alpha<2$ and defining $\delta_i^{(t)}\triangleq \frac{\vect{m}_i^{(t)}-\vect{m}_i^*}{\|\vect{m}_i^*\|^2}$ for each $i\in\{1,\sim 1\}$, we have
\begin{align*}
\frac{\mathbb{E}[\|\delta^{(t)}\|^2_2]}{\|\vect{m}^*\|^2_2}  &\leq \left(\frac{3^{\alpha}{\alpha}^2 \Lambda(\Phi',k_{lu})}{2^\alpha(\alpha-1)\nu^2}\right)\frac{1}{t}+\frac{\mathbb{E}[{\|{\delta}_{\mc{E}_{\sim 1}}^{(0)}\|}^2_2]}{\|\vect{m}_{\mc{E}_{\sim 1}}^*\|^2_2}\left(\frac{2}{t}\right)^\alpha \nonumber
\end{align*}
for all iteration $t=1,2,3,\ldots$ where $\Lambda(\Phi',k_{lu})$ is a constant which depends on some factor functions (through $\Phi'$) and some variable nodes	(through $k_{lu}$). For more details refer to Appendix.
\item[iii)] High probability bounds on error: With step size $\alpha^{(t)}=\frac{1}{\nu(t+1)}$, for any $1>\epsilon>0$ and $\forall t=1,2,\ldots,$ we have 
\begin{align*}
\delta^{(t+1)} \leq \frac{\Lambda(\Phi',k_{lu})}{\nu^2}\frac{1+\log(t+1)}{t+1}+\frac{4Q(\Phi',k_{lu})}{\nu^2\sqrt{\epsilon}}\frac{\sqrt{(1+\log(t+1))^2+4}}{t+1} \nonumber
\end{align*}
with at least probability $1-\epsilon$.
% where $Q(\Phi',k_{lu})$ is a constant which depends on kernel functions and some variable nodes.
\end{itemize}
\end{theorem}
The proof of Theorem~\ref{thm:Convergence_Results} and more discussion about the overall complexity of SGBP versus GBP can be found in Appendix.

%--------------------------------------------------------------------
\subsection{The Overall Complexity of SGBP vs. GBP}
Note that the complexity of ordinary parent-to-child algorithm is dominated by the highest-level ancestor regions with largest number of variables. Let us assume that we have $N$ highest-level ancestor regions, shown by $A_i$ for $1\leq i\leq N$. Therefore, the complexity of parents-to-child algorithm is of order ${\mc{O}}(|d|^{A_{\max}})$ where $A_{\max}=\max \big[|A|_1,\ldots,|A|_N \big]$. 

If the conditions of Theorem~\ref{thm:mainresult} hold between region $A_{\max}$ and all of its children (if there is more than one region with size $|A_{\max}|$ this statement should hold for all of them), then the SGBP algorithm (Algorithm~\ref{alg:SGBP_Alg}) will reduce the overall complexity, otherwise our algorithm will not affect the dominant computation complexity; though it may reduce the update rule complexity over some of the other edges which are not dominant in terms of complexity.
%Regarding our proposed stochastic algorithm along satisfying the conditions of Theorem~\ref{thm:mainresult}, if $|\mc{A}|_{i}\leq|\mc{A}|_j$ for all $i\neq j$, then Algorithm~\ref{alg:SGBP_Alg} will reduce the overall complexity, otherwise our algorithm will not affect the dominant computation complexity, though reducing the complexity order of second dominant message updating rules. 

Now, under the contractivity assumption of the global update function, it can be inferred that parents-to-child algorithm associated with each edge $(P\rightarrow R)$ demands $t=\mathcal{O}(\log(\frac{1}{\epsilon}))$ iteration to achieve $\epsilon$ precision, while according to Theorem~\ref{thm:Convergence_Results}, to get the same precision $\epsilon$ in SGBP algorithm, $t=\mathcal{O}(\frac{1}{\epsilon})$ iteration is needed. Nonetheless, using the Corollary~\ref{cor:orderreduction}, the computation complexity of the dominant update rule of SGBP algorithm is of order $\mc{O}(d^{(|A_{\max}|-\eta)}\frac{1}{\epsilon})$ in comparison to $\mc{O}(d^{|A_{\max}|}\log(\frac{1}{\epsilon}))$ for GBP algorithm where $\eta=\max_{C\in \mc{R}: (A_{\max},C)\in \mc{E}_r} \eta_{A_{\max}C}$. In particular, if $d>\exp\left( \frac{\log(1/\epsilon)}{\eta} \right) = \left(\frac{1}{\epsilon} \right)^{(1/\eta)}$ then SGBP leads to lower complexity than GBP to achieve the same error $\epsilon$.

\subsection{Simulation Results}
Considering a pairwise MRF, in this section we present some simulation results to study the impact of our algorithm along verifying our theoretical results. We choose the so-called Potts model (which is a generalization to Ising model; see \cite{felzenszwalb2006efficient}) of size $3\times 3$ for our simulation purpose.
% as a common model in various computer vision problems. 
We have the following potentials assigned to each of the edges $(u,v)\in \mc{E}$
\[
\psi_{uv}(i,j)=\left\{ \begin{array}{ll}
1 & \mathrm{if} \; i=j,\\
\gamma & \mathrm{Otherwise}.
\end{array}\right.
\]
where $0<\gamma<1$. For the nodes' potential we have 
\[
\phi_{u}(i)=\left\{ \begin{array}{ll}
1 & \mathrm{if} \; i=1,\\
\mu+\sigma Y & \mathrm{Otherwise}.
\end{array}\right.
\]
in which $\sigma$ and $\mu$ meet the conditions $0<\sigma\leq \mu$ and $\sigma+ \mu<1$ and $Y$ should have the uniform distribution in the span of $(-1,1)$ in addition to being independent from other parameters. We take the following steps to run our simulation. First, setting $\sigma=\mu=\gamma=0.1$, we run parent-to-child algorithm with region size of $4$
%, shown in figure \ref{fig:exageneral}, 
to get the asymptotic $m^*$. Second, with the same parameters and taking $\alpha^{(t)}=\frac{2}{(1+t)}$ for $d\in\{4,8,16,32\}$, we perform Algorithm~\ref{alg:SGBP_Alg} for the same region graph.
% in figure \ref{fig:exageneral} with distinction of using $5\times5$ grid. 
It is worth noting that to calculate $\frac{\mathbb{E}[\|\delta^{(t)}\|^2_2]}{\|\vect{m}^*\|^2_2}$, we run algorithm $20$ times and then average over error corresponding to each simulation.
As it illustrated in the simulation result of Figure~\ref{fig:FinalResult}, this result is in consistency with the Theorem~\ref{thm:Convergence_Results}. Moreover, you can also observe the running time comparison between SGBP and GBP algorithm in Figure \ref{fig:RunningtimeResult}.  

\begin{figure}
\begin{center}
\includegraphics[scale=0.65]{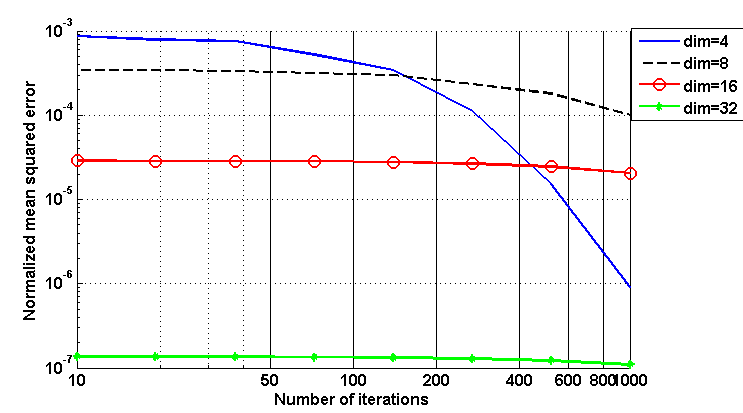}
\end{center}
\caption{The normalized mean-squared error of SGBP versus number of iterations for a Potts models of size $3\times 3$.}
\label{fig:FinalResult}
\end{figure}

\begin{figure}
\begin{center}
\includegraphics[scale=0.65]{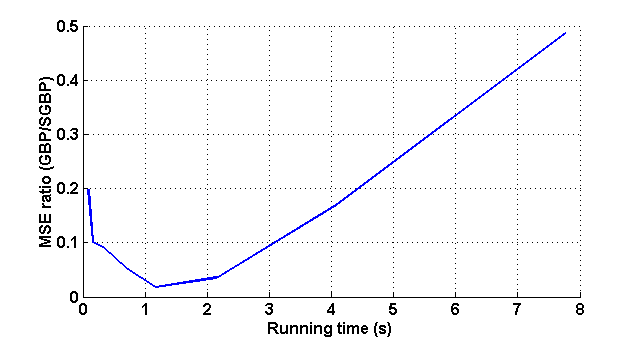}
\end{center}
\caption{The running time comparison between SGBP and GBP algorithm with $d=4$ for a Potts models of size $3\times 3$.}
\label{fig:RunningtimeResult}
\end{figure}

%--------------------------------------------------------------------
\section{Examples}
%--------------------------------------------------------------------
\subsection{A General Example}\label{ex:genex}
As mentioned before we would like to emphasize that under assumptions of the Theorem~\ref{thm:mainresult}, the complexity gain of our algorithm depends on the graph topology and choice of regions.
\begin{figure}[h]
          \centering
          \includegraphics[scale=0.50]{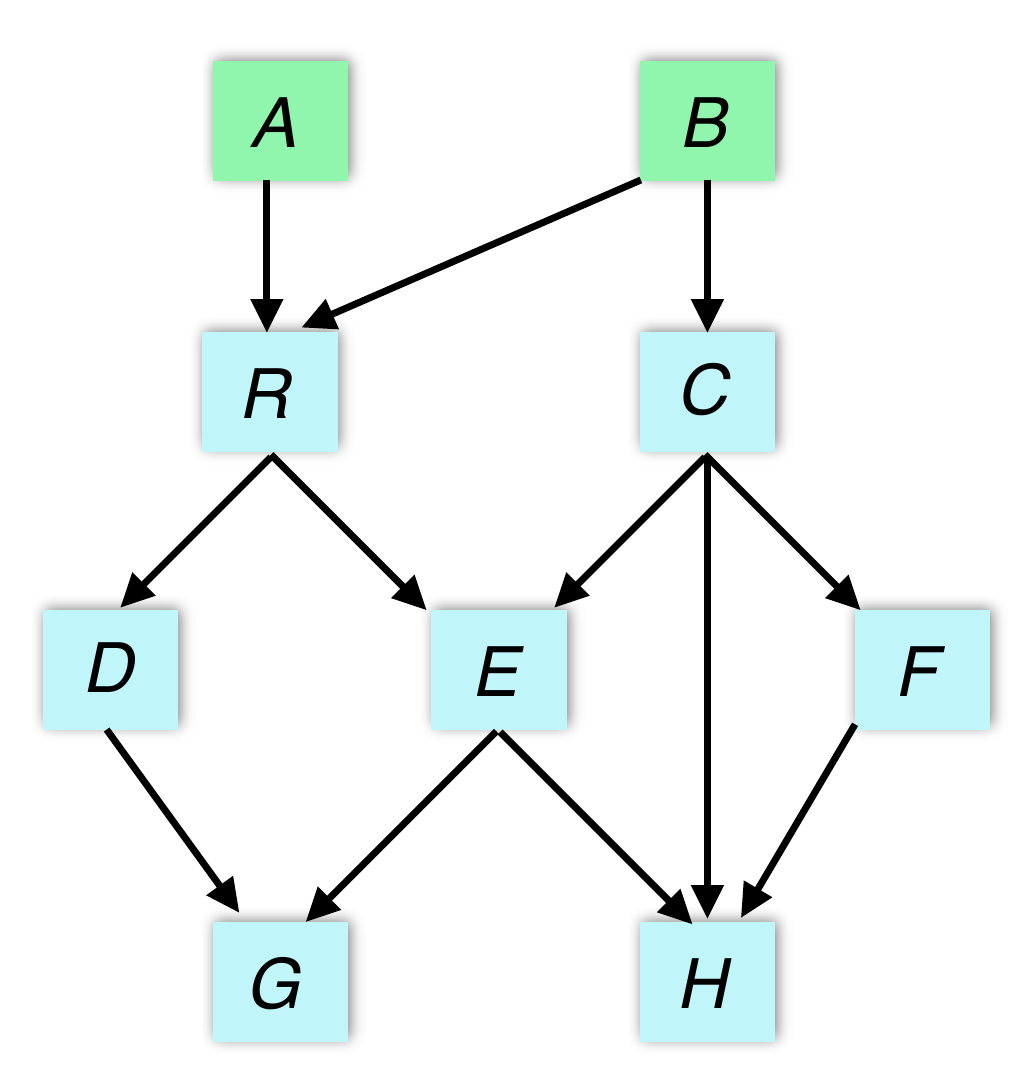}
          \caption{Graph region of an arbitrary graph with the assumption $\vect{x}_P=\vect{x}_{P'}$.}
          \label{figgencase1}
          \end{figure} 
To further clarity this idea, let's discuss the following belief equation update formulas. Moreover, we show that how we can reformulate the parent-to-child algorithm in terms of expectation. Considering Figure~\ref{figgencase1}, we have the following update rules:
\begin{align*}
m_{B\rightarrow R}(\vect{x}_R) &= c\sum_{\vect{x}_{B\setminus R}} \Phi_{B\setminus R}(\vect{x}_{B})\frac{1}{m_{C\rightarrow H}(\vect{x}_H)m_{F\rightarrow H}(\vect{x}_H)}\\
&= c\sum_{\vect{x}_{B\setminus R}} \Phi_{B\setminus R}(\vect{x}_{B})\hat{M}(\vect{x}_{T=H})\\
&=k(\vect{x}_{H\setminus(B\setminus R)}) \sum_{\vect{x}_{(B\setminus R) \setminus H}}\mathbb{E}_{{\vect{x}}_{H\cap(B\setminus R)}}[\Phi_{B\setminus R}(\vect{x}_{B}),\vect{x}_{\{B\setminus R\}}],
\end{align*}

\begin{align*}
m_{A\rightarrow R}(\vect{x}_R)
&=c\sum_{\vect{x}_{A\setminus R}} \Phi_{A\setminus R}(\vect{x}_{A}),
\end{align*}

\begin{align*}
m_{B\rightarrow C}(\vect{x}_C)
&=c\sum_{\vect{x}_{B\setminus C}} \Phi_{B\setminus C}(\vect{x}_{B})\frac{m_{A\rightarrow R}(\vect{x}_{R})}{m_{R\rightarrow E}(\vect{x}_{E})\times m_{D\rightarrow G}(\vect{x}_{E})}\\
&=c\sum_{\vect{x}_{B\setminus C}} \Phi_{B\setminus C}(\vect{x}_{B})\hat{M}(\vect{x}_{T=R})\\
&=c\sum_{\vect{x}_{B\setminus C}} \Phi_{B\setminus C}(\vect{x}_{B})\hat{M}({\vect{x}}_{R\cap (B\setminus C)},{\vect{x}}_{R\setminus (B\setminus C)})\\
&=k(\vect{x}_{R\setminus (B\setminus C)})\sum_{\vect{x}_{(B\setminus C)\setminus R}}\mathbb{E}_{{\vect{x}}_{R\cap (B\setminus C)}} \left[\Phi_{B\setminus C}(\vect{x}_{(B\setminus C)\cap R},{\vect{x}_{{(B\setminus C)\setminus R}}},{\vect{x}}_{B\setminus (B\setminus C)}) \right],
\end{align*}
and
\begin{align*}
m_{D\rightarrow G}(x_G)
&=c\sum_{\vect{x}_{D\setminus G}} \Phi_{D\setminus G}(x_{D})\frac{m_{R\rightarrow D}(\vect{x}_{D})}{m_{D\rightarrow G}(\vect{x}_{G})}\\
&=k(\vect{x}_{D\cap G})\sum_{\vect{x}_{D\setminus (D\setminus G)}} \mathbb{E}_{\vect{x}_{D\setminus G}}\left[ \Phi_{D\setminus G}(\vect{x}_{D}),{\vect{x}_{D\setminus (D\setminus G)}} \right].
\end{align*}
Furthermore we have,
\begin{align*}
m_{E\rightarrow G} &= c\sum_{\vect{x}_{E\setminus G}(\vect{x}_{E})} \Phi_{E\setminus G}(\vect{x}_{E})\left[\frac{m_{R\rightarrow E}(\vect{x}_{E})\times m_{C\rightarrow E}(\vect{x}_{E})\times m_{C\rightarrow H}(\vect{x}_{H})\times m_{F\rightarrow H}(\vect{x}_{H})}{m_{D\rightarrow G}(\vect{x}_{G})} \right]\\
&=c\sum_{\vect{x}_{E\setminus G}(\vect{x}_{E})} \Phi_{E\setminus G}(\vect{x}_{E})\hat{M}(\vect{x}_{T=E})
\end{align*}
Note that the updating rule equation for $m_{A\rightarrow R}$ cannot be rewritten in the form of an expected value due to contradicting the first condition of Theorem~\ref{thm:mainresult}.
\begin{remark}
 Considering conditions of Theorem~\ref{thm:mainresult} for Figure \ref{figgencase1}, it should be noticed that because these conditions are satisfied for $m_{B\rightarrow R}$ and $m_{B\rightarrow C}$, SGBP does help reducing complexity. However, since the first and the second condition does not hold for $m_{D\rightarrow G}$ and $m_{A\rightarrow R}$, respectively, the proposed algorithm does not improve the computation complexity of message updates over these edges.
\hfill {\small $\blacksquare$}
\end{remark}

%--------------------------------------------------------------------
\subsection{An Example in Smaller Regions} 	
In the following, we present some examples in which the impact of our algorithm in reduction of the complexity of ordinary GBP is shown. Furthermore, {we use some Matrix representation} for following examples to illustrate our idea clearly. 

\begin{example}
\begin{figure}[h]
 \begin{center}
 \begin{tikzpicture}[scale=2,very thick]
 \node at (-1.7,5.5)(index1){};
 \node at (-1.7,5) [circle,draw=black!50,fill=blue!20!white,inner sep=8pt, minimum size=8mm] (node1l) {1};
 \node at (-1.7,4) [circle,draw=black!50,fill=blue!20!white,inner sep=8pt, minimum size=8mm] (node2l) {4};
 \node at (-1.7,3) [circle,draw=black!50,fill=blue!20!white,inner sep=8pt, minimum size=8mm] (node3l) {7};
 \node at (-0.85,3.5)[circle,dashed,draw=black!50,minimum size=55mm] (node3x){};
 \node at (-0.85,4.5)[circle,dashed,draw=black!50,minimum size=55mm] (node3x){};
 \node at (+0.85,3.5)[circle,dashed,draw=black!50,minimum size=55mm] (node3x){};
  \node at (+0.85,4.5)[circle,dashed,draw=black!50,minimum size=55mm] (node3x){};
  \node at (+0.85,4.5)[circle,dashed,draw=black!50,minimum size=55mm] (node3x){}; 
 \node at (0,5.5)(index2){};
 \node at (0,5) [circle,draw=black!50,fill=blue!20!white,inner sep=8pt, minimum size=8mm] (node1) {2};
 \node at (0,4) [circle,draw=black!50,fill=blue!20!white,inner sep=8pt, minimum size=8mm] (node2) {5};
 \node at (0,3) [circle,draw=black!50,fill=blue!20!white,inner sep=8pt, minimum size=8mm] (node3) {8};
 
 \node at (1.7,5.5)(index3){};
 \node at (1.7,5) [circle,draw=black!50,fill=blue!20!white,inner sep=8pt, minimum size=8mm] (node1r) {3};
 \node at (1.7,4) [circle,draw=black!50,fill=blue!20!white,inner sep=8pt, minimum size=8mm] (node2r) {6};
 \node at (1.7,3) [circle,draw=black!50,fill=blue!20!white,inner sep=8pt, minimum size=8mm] (node3r) {9};
 \draw[thick]  (node3l)--(node3)--(node3r)--(node2r)--(node1r)--(node1)--(node1l)--(node2l)--(node2)--(node2r);\draw[thick] (node2l)--(node3l);\draw[thick](node1)--(node2)--(node3);
 
 %\draw[->,dashed,blue] (node1)--(node6r);
 \end{tikzpicture}
 \caption{Basic clusters in 9 nodes grid}\label{jamdecfig}
 \label{fig:exageneral}
 \end{center}
 \end{figure}

 \begin{figure}[h]
 \begin{center}
 \begin{tikzpicture}[scale=1.5,very thick]
 
  \node at (-4,5) [rectangle,draw=black!50,fill=blue!20!white,inner sep=10pt, minimum size=10mm] (node1ll) {1245};
  \node at (-4,4) [rectangle,draw=black!50,fill=blue!20!white,inner sep=10pt, minimum size=10mm] (node2ll) {25};
  %\node at (-4,3) [rectangle,draw=black!50,fill=blue!20!white,inner sep=8pt, minimum size=8mm] (node3ll) {2};

 \node at (-2,5) [rectangle,draw=black!50,fill=blue!20!white,inner sep=10pt, minimum size=10mm] (node1l) {2356};
 \node at (-2,4) [rectangle,draw=black!50,fill=blue!20!white,inner sep=10pt, minimum size=10mm] (node2l) {45};
 %\node at (-2,3) [rectangle,draw=black!50,fill=blue!20!white,inner sep=8pt, minimum size=8mm] (node3l) {2};

 \node at (0,5) [rectangle,draw=black!50,fill=blue!20!white,inner sep=10pt, minimum size=10mm] (node1) {4578};
 \node at (0,4) [rectangle,draw=black!50,fill=blue!20!white,inner sep=10pt, minimum size=10mm] (node2) {56};
 \node at (-1,3) [rectangle,draw=black!50,fill=blue!20!white,inner sep=8pt, minimum size=8mm] (node3) {5};

 \node at (2,5) [rectangle,draw=black!50,fill=blue!20!white,inner sep=10pt, minimum size=10mm] (node1r) {5689};
 \node at (2,4) [rectangle,draw=black!50,fill=blue!20!white,inner sep=10pt, minimum size=10mm] (node2r) {58};
 %\node at (2,3) [rectangle,draw=black!50,fill=blue!20!white,inner sep=8pt, minimum size=8mm] (node3r) {2};
 
 \draw (node1ll)--(node2ll)--(node3)--(node2l)--(node1ll);
 \draw (node1l)--(node2ll)--(node3);\draw (node1l)--(node2)--(node3); \draw (node1)--(node2r)--(node3);\draw (node1)--(node2l);\draw (node1r)--(node2);\draw (node1r)--(node2r);
 \end{tikzpicture}
 \caption{Graph region of parents to child GBP algorithm}
 \label{fig13.5general}
 \end{center}
 \end{figure}
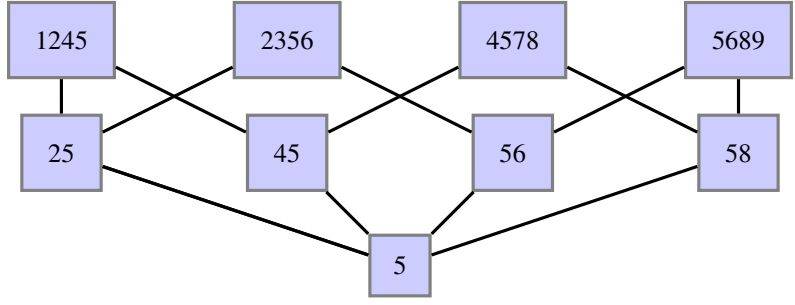
 Let's move on from these definition to consider an example in Parent-to-child generalized belief propagation algorithm and the feasibility of having low complexity stochastic one. Consider the following figure with basic clusters of four nodes, in which the belief equations for the shown region graph are:
 \begin{align}
 b_{1245}&=k[\phi_1\phi_2\phi_4\phi_5\psi_{12}\psi_{14}\psi_{25}\psi_{45}][m_{36\rightarrow 25}m_{78\rightarrow 45}m_{6\rightarrow 5}m_{8\rightarrow 5}]\nonumber\\
 b_{2356}&=k[\phi_2\phi_3\phi_5\phi_6\psi_{25}\psi_{56}\psi_{36}\psi_{23}][m_{14\rightarrow 25}m_{89\rightarrow 56}m_{8\rightarrow 5}m_{4\rightarrow 5}]\nonumber\\
 b_{25}&=k[\phi_2\phi_5\psi_{25}][m_{14\rightarrow 25}m_{36\rightarrow 25}m_{4\rightarrow 5}m_{6\rightarrow 5}m_{8\rightarrow 5}]\nonumber\\
 b_{45}&=k[\phi_4\phi_5\psi_{45}][m_{12\rightarrow 45}m_{78\rightarrow 45}m_{2\rightarrow 5}m_{6\rightarrow 5}m_{8\rightarrow 5}]\nonumber\\
 b_{5}&=k[\phi_{5}][m_{2\rightarrow 5}m_{4\rightarrow 5}m_{6\rightarrow 5}m_{8\rightarrow 5}]
 \end{align}
 Then recalling our definitions in previous sections, we can derive message updating rules as follows:
 \begin{align}
 m_{12\rightarrow 45}&=c\sum_{x_1,x_2\in{{\cal{X}}}}\phi_1(x_1)\phi_2(x_2)\psi_{12}(x_1,x_2)\psi_{14}(x_1,x_4)\psi_{25}(x_2,x_5)\frac{m_{36\rightarrow 25}(x_2,x_5)}{m_{2\rightarrow 5}(x_5)}\nonumber\\
 &=c\sum_{x_1,x_2\in{{\cal{X}}}}\Phi_{1245}(x_1,x_2,x_4,x_5)\hat{M}_{T=(2,5)}(x_2,x_5)\nonumber\\
 &=k_1(x_5)\sum_{x_1}\mathbb{E}_{x_2\sim\hat{M}_{T}}\Phi_{1245}(x_1,J,x_4,x_5)\label{eq1}\\
  m_{14\rightarrow 25}&=c\sum_{x_1,x_4\in{{\cal{X}}}}\phi_1(x_1)\phi_4(x_4)\psi_{14}(x_1,x_4)\psi_{12}(x_1,x_2)\psi_{45}(x_4,x_5)\frac{m_{78\rightarrow 45}(x_4,x_5)}{m_{4\rightarrow 5}(x_5)}\nonumber\\
  &=c\sum_{x_1,x_4\in{{\cal{X}}}}\phi_1(x_1)\phi_4(x_4)\psi_{14}(x_1,x_4)\psi_{12}(x_1,x_2)\psi_{45}(x_4,x_5)\frac{m_{78\rightarrow 45}(x_4,x_5)}{m_{4\rightarrow 5}(x_5)}\nonumber\\
   &=c\sum_{x_1,x_4\in{{\cal{X}}}}\Phi_{1425}(x_1,x_2,x_4,x_5)\hat{M}_{T=(4,5)}(x_4,x_5)\nonumber\\
   &=k_2(x_5)\sum_{x_1}\mathbb{E}_{x_4\sim\hat{M}_{T}}\Phi_{1425}(x_1,x_2,J,x_5)\label{eq2}
 \end{align}
 where $k_1(x_5)=c\sum_{x_2}\hat{M}(x_2,x_5)$, $k_2(x_5)=c\sum_{x_4}\hat{M}(x_4,x_5)$, \begin{align*}\Phi_{1245}(x_1,x_2,x_4,x_5)=&\phi_1(x_1)\phi_2(x_2)\psi_{12}(x_1,x_2)\psi_{14}(x_1,x_4)\psi_{25}(x_2,x_5)\\
 \Phi_{1425}(x_1,x_2,x_4,x_5)=&\phi_1(x_1)\phi_4(x_4)\psi_{14}(x_1,x_4)\psi_{12}(x_1,x_2)\psi_{45}(x_4,x_5)
 \end{align*} and defining 
 \begin{align*}
 \Phi_{25}(x_2,x_5)=&\phi_2(x_2)\psi_{25}(x_2,x_5)\\
  \Phi_{45}(x_4,x_5)=&\phi_4(x_4)\psi_{45}(x_4,x_5)
 \end{align*} 
 we have
 \begin{align}
\label{eq3} m_{2\rightarrow 5}=c\sum_{x_2\in{\cal{X}}}[\phi_2(x_2)\psi_{25}(x_2,x_5)]\big[{m_{14\rightarrow 25}(x_2,x_5)m_{36\rightarrow 25}(x_2,x_5)}\big]\\
\label{eq4} m_{4\rightarrow 5}=c\sum_{x_4\in{\cal{X}}}[\phi_4(x_4)\psi_{45}(x_4,x_5)]\big[{m_{12\rightarrow 45}(x_4,x_5)m_{76\rightarrow 45}(x_4,x_5)}\big]\\\nonumber
 \end{align}
 Note that the equations \ref{eq3}, \ref{eq4} do not meet the requirements of theorem \ref{thm:mainresult}, so we use them without any alteration. However, we cam apply stochastic updating with equations \ref{eq1}, \ref{eq2}. 

 Therefore, the distribution of random index generation associated with equations \ref{eq1} and \ref{eq2} is   \begin{align}p(j|i)=\frac{\hat{M}(j,i)}{\sum_{j'}{\hat{M}(j',i)}}\label{2node}\end{align} %\begin{align}p_{uv}(j|i)=\frac{\tilde{m}_{uv}(i,j)}{\sum_m\tilde{m}_{uv}(i,m)}\label{1node}\end{align} from equations \ref{eqf1} and \ref{eqf2}  form a probability mass function that depends on the update matrix as well as the messages. Thus, the massage passed along each edge in the region graph can be viewed as the expectation of the normalized matrix and columns respectively by $\big[\frac{\big[\Gamma_{14}\times\Gamma^T_{25}\big](:,i)}{\delta(n)}\big]$ and $\frac{\Gamma_{uv}(:,j)}{\delta(j)}$ for $i,j=1,\cdots,d$, we have \begin{align}m^{(t+1)}(:,i)=\mathbb{E}_{{J}\sim p(j|i)}\big[\frac{\big[\Gamma\times\Gamma^T\big](:,J)}{\delta(J)}\big]\end{align}and \begin{align}m^{(t+1)}(i)=\mathbb{E}_{{J}\sim p(j|i)}\big[\frac{\Gamma(i,J)}{\tilde{\delta}}\big]\end{align}
 so we can introduce our algorithm now as follows:
 \begin{itemize}
 \item For the message from four-node clusters to two-node-cluster, pick two indexes $i$ and $j$ with probabilities according to distribution \ref{2node}. Then, the updating rules for these nodes is:
\begin{align}
m_{12\rightarrow 45}^{(t+1)}&=(1-\alpha^{(t)})m_{12\rightarrow 45}^{(t)}+\alpha^{(t)} k_1(x_5)\sum_{x_1}\Phi(x_1,J,x_4,x_5)\\
m_{14\rightarrow 25}^{(t+1)}&=(1-\alpha^{(t)})m_{14\rightarrow 25}^{(t)}+\alpha^{(t)} k_2(x_5)\sum_{x_1}\Phi(x_1,x_2,J,x_5)
\end{align} where $\alpha^{(t)}=\mathcal{O}(\frac{1}{t})$ is some step size.
 \item For the message from middle nodes to the node $5$, we do the same as following:
 
 \begin{align}
 m_{2\rightarrow 5}^{(t+1)}&=(1-\alpha^{(t)})m_{2\rightarrow 5}^{(t)}+\alpha^{(t)} c\:\Phi(J_{25},x_5)\\
 m_{4\rightarrow 5}^{(t+1)}&=(1-\alpha^{(t)})m_{4\rightarrow 5}^{(t)}+\alpha^{(t)} c\:\Phi(J_{45},x_5)
 \end{align}
 %\item For the message from two-node clusters to one node cluster, pick a indexes $j$ with probability according to the distribution \ref{1node}. Then, the updating rules for these nodes is:
% \begin{align}m^{(t+1)}(i)=(1-\beta^{(t)})m^{(t)}(i)+\beta^{(t)} \frac{\Gamma(i,J)}{\tilde{\delta}}\end{align} where $\delta^{(t)}$ is some step size.
 \item So it can be seen that the complexity of calculating the distribution and the updating rule is of ${\cal{O}}(d^2)$ and ${\cal{O}}(d^3)$ respectively, which is less than ${\cal{O}}(d^4)$, the complexity of ordinary GBP. 
 \end{itemize}
 \end{example}

%\section{Conclusion}

%\section*{Acknowledgment}

%\iffalse
\appendix
\section{Proofs}
In this section, we give a proof of Theorem~\ref{thm:Convergence_Results} which is similar to the techniques employed in \cite{noorshams_stochastic_2013} to show the convergence of SBP. However, we have to adapt the proof technique properly to work for the SGBP algorithm.

Before stating the proof, we state some general observations in the following  that we will be used later in our proof.

\begin{lemma}\label{lem:Graph-region property}
For any region graph $G_r$ and each edge $(P\rightarrow R)\in \mc{E}_r$, except the ones coming out from highest level ancestor regions, 
%in the $m_{P\rightarrow R}$ 
we must have $(P\setminus R)\subset T_{PR}$, which means $T_{PR} \cap (P\setminus R)\neq\varnothing$. 
\end{lemma}
%--------
\begin{proof}
Consider an arbitrary edge $(P\rightarrow R)\in \mc{E}_r$ in the region graph $G_r$. Recalling the update rule \eqref{eq:UpdateRule_FirstEquation} over the edge $(P\rightarrow R)$  
\begin{align}
m_{P\rightarrow R}(\vect{x}_{R}) =& \sum_{\vect{x}_{P\setminus R}}\Phi_{P\setminus R}(\vect{x}_{P'})\frac{\prod_{(I,J)\in N(P,R)}m_{I\rightarrow J}(\vect{x}_{J})}{\prod_{(I,J)\in D(P,R)}m_{I\rightarrow J}(\vect{x}_{J})} = \sum_{\vect{x}_{P\setminus R}}\Phi_{P\setminus R}(\vect{x}_{P'})\hat{M}(\vect{x}_{T_{PR}}) \nonumber
\end{align}
we notice that for the non-highest-level-ancestor regions there must be at least one incoming edge from its parents. This means that in the nominator of $\hat{M}(\vect{x}_{T_{PR}})$ at least one message $m_{I\rightarrow P}$ must appear that depends on the variables $\vect{x}_P$, while such a message does not present in the denominator according to the definition of parent-to-child GBP algorithm, \ie, see \eqref{eq:UpdateRule_FirstEquation}, \eqref{eq:Def_Set_N(P,R)} and \eqref{eq:Def_Set_D(P,R)}. Therefore, $T_{PR}=P$ and hence $P\setminus R \subset T_{PR}$.
%$ \hat{M}(\vect{x}_{T_{PR}})=\hat{M}(\vect{x}_{P})$.
\end{proof}

\begin{remark}\label{remrk:SGBP_UpdateRuleCategory_non-AncestorRegion}
Combining the results of Lemma~\ref{lem:Graph-region property} and Theorem~\ref{thm:mainresult}, we conclude that for each regions $P\in \set{R}$, except for the highest-level ancestor regions, the update rule \eqref{eq:UpdateRule_Alg} of SGBP over the outgoing edges is stochastic (\ie, $\vect{J}^{(t+1)}_{PR}$ is not a trivial random variable).
It should be emphasized that since the condition (ii) of Theorem~\ref{thm:mainresult} is violated for such regions, we will not obtain any complexity gain by applying stochastic update rule over such an edge.
However, since the overall complexity of SGBP algorithm is determined by the largest highest-level ancestor regions, it is not harmful to have no complexity gain in lower-level regions.
%We can conclude from Lemma~\ref{lem:Graph-region property} along the conditions of Theorem~\ref{thm:mainresult} that excluding the message updating rules for ancestor regions in region graph, stochastic algorithm can be used in those edges. Note that in the corresponding message updating rule only the second condition of Theorem~\ref{thm:mainresult} is contracted, meaning that there is no benefit in terms of complexity reduction. 
\hfill {\small $\blacksquare$}
\end{remark}

\begin{remark}\label{remrk:SGBP_UpdateRuleCategory_AncestorRegion}
Considering the highest-level ancestor regions, the outgoing message update rule in SGBP can be categorized into three different groups. Let $P\in \set{R}$ be a highest-level ancestor region that sends a message to one of its children region $R$. Then, the following cases can be recognized.
\begin{enumerate}
\item $T_{PR} = \varnothing$: 
In this case the update rule \eqref{eq:UpdateRule_FirstEquation} of parent-to-child GBP becomes independent of previous iteration messages. Hence, the update rule of SGBP (Equation~\ref{eq:UpdateRule_Alg}) is reduced to the following expression
\begin{equation}\label{eq:SGBP_UpdateRuleCategory_Case1}
m_{P\rightarrow R}^{(t+1)} (\vect{x}_R) = (1-\alpha^{(t)})m_{P\rightarrow R}^{(t)} (\vect{x}_R) +\alpha^{(t)}      \sum_{\vect{x}_{P\setminus R}}   \Phi_{P\setminus R} \left(\vect{x}_{P'} \right). 
\end{equation}
Notice that the second term on the right hand side of  \eqref{eq:SGBP_UpdateRuleCategory_Case1} remains the same for all iterations. 

For such edges in the region graph, we assume that the initial messages are the same as update messages. Hence, for all value of $t$ we have
\[
m_{P\rightarrow R}^{(t+1)} (\vect{x}_R) = \sum_{\vect{x}_{P\setminus R}}   \Phi_{P\setminus R} \left(\vect{x}_{P'} \right)
\]
which does not depend on $t$.
\item $(P\setminus R)\cap T_{PR} = \varnothing$ but $T_{PR}\neq \varnothing$:
Similar to the previous case, in this case, the update rule of SGBP over the edge $(P\rightarrow R)$ is also deterministic as stated in the following
\begin{align}\label{eq:SGBP_UpdateRuleCategory_Case2}
m_{P\rightarrow R}^{(t+1)} (\vect{x}_R) &= (1-\alpha^{(t)})m_{P\rightarrow R}^{(t)} (\vect{x}_R) + \alpha^{(t)}  k_{PR}^{(t)} (\vect{x}_{T_{PR}})  \left[\sum_{\vect{x}_{P\setminus R}}\Phi_{P\setminus R} \left( \vect{x}_{P'} \right)\right] \nonumber\\
&= (1-\alpha^{(t)})m_{P\rightarrow R}^{(t)} (\vect{x}_R) + \alpha^{(t)}  \tilde{\Phi}_{P\setminus R} \hat{M}^{(t)}(\vect{x}_{T_{PR}})
\end{align}
where $\tilde{\Phi}_{P\setminus R} \triangleq \sum_{P\setminus R}\Phi_{P\setminus R}(\vect{x}_{P'})$ which is a constant.
\item $(P\setminus R)\cap T_{PR} \neq \varnothing$:
In this case, the update rule of SGBP over the edge $(P\rightarrow R)$ is stochastic (\ie, $\vect{J}^{(t+1)}_{PR}$ is not a trivial random variable) and is stated in \eqref{eq:UpdateRule_Alg}.
\end{enumerate}
According to the above discussion, the deterministic message update rules can only happen over the outgoing edges of the highest-level ancestor regions.
For further clarification, you can refer to \S\ref{ex:genex} where $m_{A\rightarrow R}(\vect{x}_{R})$ illustrates a deterministic message updating rule.
\hfill {\small $\blacksquare$}
\end{remark}

%\begin{lemma}
%According to the parent-to-child GBP algorithm, deterministic message update rules can only happen within ancestor regions' message updating rules.
%\end{lemma}
%%--------
%\begin{proof}
%Reminding the conditions of Theorem~\ref{thm:mainresult}, it can be seen that if $(P\setminus R)\cap T_{PR}= \varnothing$ there can not be a random choice for the alphabets to be used in our algorithm. Therefore, considering that $(P\setminus R)\neq\varnothing$, if $(P\setminus R)\cap T_{PR}= \varnothing$, we can conclude that $T_{PR}=\varnothing$. On the other hand, using the Lemma~\ref{lem:Graph-region property} it can be inferred that as for the non-ancestor regions $P=T$, ancestor nodes remain to hold for ancestor regions.
%For further clarification, you can refer to \S\ref{ex:genex} where $m_{A\rightarrow R}(\vect{x}_{R})$ illustrates a deterministic message updating rule.
%\end{proof}

\vspace{8pt}
\begin{proof}[Proof of Theorem~\ref{thm:Convergence_Results}]
First, notice that the existence and uniqueness of GBP fixed point under the contractivity assumption of \eqref{eq:ContractiveCond_Def} can be deduced by applying the Banach fixed point theorem \cite{agarwal2001fixed,noorshams_stochastic_2013}.
%\begin{remark}
%Note that the updating rule $m^{(t)}_{P\rightarrow R}$ in \ref{alg:SGBP_Alg} is an scalar function, while for our purpose of proof in this section, we use the $\vect{m}^{(t)}_{P\rightarrow R}$ to indicate the vector form of the function $m^{(t)}_{P\rightarrow R}$.
%\end{remark}
Then, following a similar approach used in \cite{noorshams_stochastic_2013}, the asymptotic convergence of SGBP algorithm can be proved by applying a version of the Robbins-Monro theorem \cite{agarwal2001fixed}, discussed in \cite[Appendix~C]{ noorshams_stochastic_2013} in detail.
\begin{itemize}
\item{\textit{Asymptotic convergence of SGBP (proof of part (i)):}}\\
Let us consider the partitioning of messages according to Remark~\ref{remrk:SGBP_UpdateRuleCategory_non-AncestorRegion} and Remark~\ref{remrk:SGBP_UpdateRuleCategory_AncestorRegion} as\footnote{Here we have removed the time index on messages for simplifying the notation.} $\vect{m}=(\vect{m}_{\mc{E}_{1}},\vect{m}_{\mc{E}_{2}},\vect{m}_{\mc{E}_{3}})$ where $\mc{E}_{1}$, $\mc{E}_{2}$ and $\mc{E}_{3}$ are a partition of edges in the region graph in which the message updating rules is independent of messages (which is also deterministic), deterministic and stochastic, respectively.
%\begin{remark}\label{rmk:Assym-proof}
%It is noteworthy that ${m}_{\mc{E}_{1}}$ only depends on edges and regions factor (potentials) and is deterministic. 
%\hfill $\blacksquare$
%\end{remark}
   
Now, to prove the asymptotic convergence of SGBP, we apply the Robbins-Monro theorem. To this end, $\forall i\in\{1,2,3\}$, we define $Y^i_{PR}(\vect{x}_R)$ as follows
\begin{align*}
 Y^{(t)}_{1,PR}(\vect{x}_{R}) &\stackrel{}{\triangleq} \sum_{\vect{x}_{P\setminus R}}\Phi_{P\setminus R}(\vect{x}_{P'}) \\
 Y^{(t)}_{2,PR}(\vect{x}_{R}) &\stackrel{}{\triangleq} \tilde{\Phi}_{P\setminus R} \hat{M}^{(t)}(\vect{x}_{T_{PR}})\\
 Y^{(t)}_{3,PR}(\vect{x}_{R}) &\triangleq \left[ k_{PR}^{(t)} \left( \vect{x}_{T\setminus (P\setminus R)  } \right) \sum_{\vect{x}_{(P\setminus R)\setminus T}} \Phi_{P\setminus R} \left(\vect{J}_{PR}^{(t+1)}, \vect{x}_{(P\setminus R)\setminus T}, \vect{x}_{P'\setminus (P\setminus R)} \right) \right],
\end{align*}
where only one of the above rules is applied for each edge $(P\rightarrow R)$ in the region graph.
Next, we need to rewrite the update function in a form that enables us to use the Robbins-Monro theorem. So, by defining $\vect{U}^{(t)}_{i,PR}\triangleq\big[ \vect{m}_{P\rightarrow R}^{(t)}- {\cal{L}}(Y^{(t)}_{i,PR}(\vect{x}_R)) \big]$, we have
\begin{align*}
   \vect{m}_{P\rightarrow R}^{(t+1)} &= \vect{m}_{P\rightarrow R}^{(t)} -\alpha^{(t)} \left[ \vect{m}_{P\rightarrow R}^{(t)}- {\cal{L}}(Y^{(t)}_{i,PR}(\vect{x}_R)) \right] \\
   &= \vect{m}_{P\rightarrow R}^{(t)} -\alpha^{(t)} \vect{U}^{(t)}_{i,PR} \left(\vect{m}_{P\rightarrow R}^{(t)},\vect{J}_{PR}^{(t+1)} \right) 
\end{align*}
in which for every fixed value of $\vect{J}_{PR}^{(t+1)} \in {\set{X}^{| T_{PR} \cap (P\setminus R)|}}$, we consider $\vect{U}_{i,PR}^{(t)}$ as a mapping from $\mathbb{R}^{d^{|R|}}$ to $\mathbb{R}^{d^{|R|}}$. By concatenating all of these functions we get the function  $\vect{U}(\cdot ,\vect{J}^{(t+1)}) : \mathbb{R}^{\Delta_{\mc{E}}} \mapsto \mathbb{R}^{\Delta_{\mc{E}}}$, where ${\Delta_{\mc{E}}}=\sum_{(P,R)\in \set{E}_r} d^{|R|}$ and $\vect{J}^{(t+1)}\in \prod_{(P\rightarrow R)\in\set{E}_3}{\set{X}^{| T_{PR} \cap (P\setminus R)|}}$. Here the product denotes for the Cartesian product. 
 
Now, we are ready to apply the Robins-Monro theorem. Using above definitions, the global update function can be rewritten as following
\begin{align}\label{eq:MUp}
\vect{m}^{(t+1)} = \vect{m}^{(t)} - \alpha^{(t)} \vect{U}(\vect{m}^{(t)},\vect{J}^{(t+1)}),\end{align}
for $t = 1,2,\ldots$. Defining the mean vector field
$\vect{u}(\vect{m})\triangleq \mathbb{E}[\vect{U}(\vect{m}, \vect{J})|\vect{m}] = \vect{m}-\vect{U}(\vect{m})$, we need only to verify that the fixed point $\vect{m}^*$ satisfies the condition
$\sup_{\vect{m}}\langle  \vect{m}-\vect{m}^* , \vect{u}(\vect{m}) \rangle> 0$, where $\langle \cdot , \cdot \rangle$ denotes the Euclidean inner product. Using the Cauchy-Schwartz inequality and the fact that $\Upsilon(\vect{m})$ is Lipschitz with constant $L=1-\frac{\nu}{2}$, for all $\vect{m}\neq \vect{m}^*$ we have 
\begin{align}
\left\langle  \vect{m}-\vect{m}^* , \vect{u}(\vect{m})-\vect{u}(\vect{m}^*) \right\rangle &= \|\vect{m}-\vect{m}^*\|^2_2- \left \langle  \vect{m}-\vect{m}^* , \vect{U}(\vect{m})-\vect{U}(\vect{m}^*) \right\rangle\nonumber\\
&\geq \frac{\nu}{2}\|\vect{m}-\vect{m}^*\|^2_2\nonumber\\ 
&> 0. \nonumber
\end{align}
Since $\vect{m}^*$ is a fixed point, we must have $\vect{u}(\vect{m}^*) = \vect{m}^* - \vect{U}(\vect{m}^*) = 0$, which concludes the proof of this part.
\vspace{8pt}
\item{\textit{Non-asymptotic bounds on the mean-square error (proof of part (ii)):}}\\
Before proceeding to the proof, we first upper bound the normalized mean error as following:
\begin{align*}
\frac{\mathbb{E}[\delta^{(t)}]}{\|\vect{m}^*\|^2} &= \frac{\mathbb{E} \left[ \|\vect{m}^{(t)}-{\vect{m}^*}\|_2^2 \right]}{\|\vect{m}^*\|^2}\\
&\stackrel{(a)}{=} \frac{\mathbb{E} \left[ \|\vect{m}_{\mc{E}_{1}}^{(t)}+\vect{m}_{\mc{E}_{\sim 1}}^{(t)}-{\vect{m}_{\mc{E}_{1}}^*}-{\vect{m}_{\mc{E}_{\sim 1}}^*}\|_2^2 \right]}{\|\vect{m}^*\|^2}\\
&\stackrel{(b)}{=} \frac{\mathbb{E}\left[ \|\vect{m}_{\mc{E}_{\sim 1}}^{(t)}-{\vect{m}_{\mc{E}_{\sim 1}}^*}\|_2^2 \right]}{\|\vect{m}^*\|^2}\\
&\stackrel{(c)}{\leq} \frac{\mathbb{E} \left[\|\vect{m}_{\mc{E}_{\sim 1}}^{(t)}-{\vect{m}_{\mc{E}_{\sim 1}}^*}\|_2^2 \right]}{\|\vect{m}_{\mc{E}_{\sim 1}}^*\|^2}\\
&= \frac{\mathbb{E} \left[\delta^{(t)}_{\mc{E}_{\sim 1}}\right]}{\|\vect{m}_{\mc{E}_{\sim 1}}^*\|^2}
\end{align*}
where in (a) we divide the expectation over the two set of edges, namely,
$\vect{m}^{(t)}_{\mc{E}_{1}}$ and $\vect{m}^{(t)}_{\set{E}_{\sim 1}}= (\vect{m}^{(t)}_{\set{E}_2},\vect{m}^{(t)}_{\mathcal{E}_3})$, (b) follows by Remark \ref{remrk:SGBP_UpdateRuleCategory_AncestorRegion}, Item~1, which implies $\vect{m}_{\set{E}_{1}}^{(t)}=\vect{m}^*_{\mc{E}_{1}}$ for all $t$, and (c) is due to the fact that $\vect{m}=(\vect{m}_{\mc{E}_{1}},\vect{m}_{\mc{E}_{\sim 1}})$. Hence, to upper bound $\mathbb{E}[{\delta}^{(t)}]$ for all $t = 1, 2, \ldots,$ we upper bound $\frac{\mathbb{E}[\delta^{(t)}_{\mc{E}_{\sim 1}}]}{\|\vect{m}_{\mc{E}_{\sim 1}}^*\|^2}$. First, we bound the quantity $\mathbb{E}[{\delta}_{\mc{E}_{\sim 1}}^{(t+1)}] - \mathbb{E}[{\delta}_{\mc{E}_{\sim 1}}^{(t)}]$ that corresponds to the increment in the mean-squared error.
 
Considering the update equation \eqref{eq:MUp} and by applying basic properties of the expectation, we obtain
\begin{align}
\mathbb{E} \left[ {\delta}_{\mc{E}_{\sim 1}}^{(t+1)} \right] - \mathbb{E} \left[ {\delta}_{\mc{E}_{\sim 1}}^{(t)} \right] &= \frac{\mathbb{E} \left[ \left\langle \vect{m}_{\mc{E}_{\sim 1}}^{(t+1)}-\vect{m}_{\mc{E}_{\sim 1}}^{(t)},\vect{m}_{\mc{E}_{\sim 1}}^{(t+1)}+\vect{m}_{\mc{E}_{\sim 1}}^{(t)}-2\vect{m}_{\mc{E}_{\sim 1}}^*\right\rangle \right]}{\|\vect{m}_{\mc{E}_{\sim 1}}^*\|^2_2}\nonumber \\
&= {(\alpha^{(t)})}^2\frac{\mathbb{E} \left[ \|U(\vect{m}_{\mc{E}_{\sim 1}}^{(t)},\vect{J}^{(t+1)})\|^2_2 \right]}{\|\vect{m}_{\mc{E}_{\sim 1}}^*\|^2_2}\nonumber\\
&\quad -2{(\alpha^{(t)})}\frac{\mathbb{E} \left[ \mathbb{E} \left[ \left\langle \vect{U}(\vect{m}^{(t)}_{\mc{E}_{\sim 1}},\vect{J^{(t+1)}}),\vect{m}_{\mc{E}_{\sim 1}}^{(t)}-\vect{m}_{\mc{E}_{\sim 1}}^* \right\rangle|\mc{F}^{(t)} \right] \right]}{\|\vect{m}_{\mc{E}_{\sim 1}}^*\|^2_2} \nonumber\\
&={(\alpha^{(t)})}^2\frac{\mathbb{E} \left[ \|U(\vect{m}_{\mc{E}_{\sim 1}}^{(t)},\vect{J}^{(t+1)})\|^2_2 \right]}{\|\vect{m}_{\mc{E}_{\sim 1}}^*\|^2_2}\nonumber\\
&\quad -2{(\alpha^{(t)})}\frac{\mathbb{E} \left[ \left\langle u(\vect{m}_{\mc{E}_{\sim 1}}^{(t)})-u(\vect{m}_{\mc{E}_{\sim 1}}^*),\vect{m}_{\mc{E}_{\sim 1}}^{(t)}-\vect{m}_{\mc{E}_{\sim 1}}^* \right\rangle \right]}{\|\vect{m}_{\mc{E}_{\sim 1}}^*\|^2_2}, \label{eq:divup}
\end{align}
where we have used $\mathbb{E}\big[\vect{U}(\vect{m}_{\mc{E}_{\sim 1}}^{(t)},\vect{J}^{(t+1)})|\mc{F}^{(t)}\big]=\vect{u}(\vect{m}_{\mc{E}_{\sim 1}}^{(t)})$ and $\vect{u}(\vect{m}^*)=0$. Moreover, we define the $\sigma$-field $\mc{F}^{(t)}\triangleq \sigma(\vect{m}^{0},\vect{m}^{1},\ldots,\vect{m}^{t})$.
From this point, to upper bound \eqref{eq:divup}, we bound each term in \eqref{eq:divup} separately. We continue the proof by upper bounding 
\[
H_1=\frac{\left\| \vect{U}(\vect{m}_{\mc{E}_{\sim 1}}^{(t)},\vect{J}^{(t+1)}) \right\|^2_2}{\|\vect{m}_{\mc{E}_{\sim 1}}^* \|^2_2}
\]
and then lower bounding 
\[
H_2=\frac{\left\langle \vect{u}(\vect{m}_{\mc{E}_{\sim 1}}^{(t)})-\vect{u}(\vect{m}_{\mc{E}_{\sim 1}}^*),\vect{m}_{\mc{E}_{\sim 1}}^{(t)}-\vect{m}_{\mc{E}_{\sim 1}}^* \right\rangle}{\|\vect{m}_{\mc{E}_{\sim 1}}^*\|^2_2}.
\]
Recall that for $(P\rightarrow R)\in\set{E}_2$ we have the following update rule
\[
\vect{m}_{P \rightarrow R}^{(t+1)}(\vect{x}_R) = \tilde{\Phi}_{P\setminus R} \hat{M}^{(t)}(\vect{x}_{T_{PR}})
\]
and for $(P\rightarrow R)\in \set{E}_3$ we have
\[
\vect{m}_{P \rightarrow R}^{(t+1)} (\vect{x}_R) = k_{PR}^{(t)} \left( \vect{x}_{T_{PR}\setminus (P\setminus R)  } \right) \sum_{\vect{x}_{(P\setminus R) \setminus  T_{PR}} } \mathbb{E}_{[\vect{X}_{(P\setminus R)\cap T_{PR}\sim \hat{Q}^{(t)} }]} \Big[ \Phi_{P\setminus R} \big(\vect{X}_{(P\setminus R)\cap T_{PR}}, \vect{x}_{(P\setminus R)\setminus T_{PR}}, \vect{x}_{P'\setminus (P\setminus R)} \big) \Big].
 \]
Hence, we have
% \begin{align*}
 %LB_{(P,R)\in\mc{E}_{2}}(\vect{x}_R)&\triangleq \min_{\vect{x}_{R}}\vect{m}_{(P,R)\in\mc{E}_2}(\vect{x}_{R})\\
 %&\leq \vect{m}_{(P,R)\in\mc{E}_2}(\vect{x}_{R})\\
 %&\leq \max_{\vect{x}_{R}}\vect{m}_{(P,R)\in\mc{E}_2}(\vect{x}_{R})\\
 %&\triangleq  UB_{(P,R)\in\mc{E}_{2}}(\vect{x}_R)
% \end{align*}
\begin{align*}
\mathrm{LB}_{(P,R)\in\mc{E}_{3}}(\vect{x}_R) &\triangleq k_{l}\left( \vect{x}_{T_{PR}\setminus (P\setminus R)  } \right) \sum_{\vect{x}_{(P\setminus R)\setminus T_{PR}}} \min_{\vect{x}_{(P\setminus R)\cap T_{PR}}}\Big[ \Phi_{P\setminus R} \big(\vect{x}_{(P\setminus R)\cap T_{PR}}, \vect{x}_{(P\setminus R)\setminus T_{PR}}, \vect{x}_{P'\setminus (P\setminus R)} \big) \Big]\\
&\leq \vect{m}_{PR}^{(t)}(\vect{x}_R)\\
&\leq k_{u}\left( \vect{x}_{T_{PR}\setminus (P\setminus R)  } \right) \sum_{\vect{x}_{(P\setminus R)\setminus T_{PR}}}\max_{\vect{x}_{(P\setminus R)\cap T}}\Big[ \Phi_{P\setminus R} \big(\vect{x}_{(P\setminus R)\cap T}, \vect{x}_{(P\setminus R)\setminus T}, \vect{x}_{P'\setminus (P\setminus R)} \big) \Big]\\
&\triangleq \mathrm{UB}_{(P,R)\in\mc{E}_{3}}(\vect{x}_R) 
\end{align*}
for all $\vect{x}_R\in {\cal{X}}^{|R|}$ due to the convex combination nature of definition of $\vect{m}_{P\rightarrow R}^{(t+1)}$. As a result, $H_1$ can be bounded as follows
\begin{align}
H_1&=\frac{\|\vect{U}(\vect{m}_{\set{E}\sim 1}^{(t)},\vect{J}^{(t+1)})\|^2_2}{\|\vect{m}_{\set{E}\sim 1}^*\|^2_2}\\
&\leq \frac{2\sum_{P\rightarrow R}(\|\vect{m}_{P\rightarrow R}^{(t)}\|_2^2+\|{\cal{L}}(Y_{PR}^{(t)}(\vect{x}_R))\|_2^2)}{\| \vect{m}^*_{P\rightarrow R}\|_2^2}\nonumber\\
&=\frac{2\big[\sum_{(P, R)\in\mc{E}_2}(\|\vect{m}_{P\rightarrow R}^{(t)}\|_2^2+\|{\cal{L}}(Y_{2,PR}^{(t)}(\vect{x}_R))\|_2^2)+\sum_{(P, R)\in\mc{E}_3}(\|\vect{m}_{P\rightarrow R}^{(t)}\|_2^2+\|{\cal{L}}(Y_{3,PR}^{(t)}(\vect{x}_R))\|_2^2)\big]}{\| \vect{m}^*_{P\rightarrow R}\|_2^2}\nonumber\\
&\leq \frac{2\max_{\vect{x}_R}\sum_{(P, R)\in\mc{E}_2}(\|\vect{m}_{P\rightarrow R}^{(t)}\|_2^2+\|{\mc{L}}(Y_{2,PR}^{(t)}(\vect{x}_R))\|_2^2)+4\sum_{(P,R)\in\mc{E}_3}(\max_{\vect{x}_R} \mathrm{UP}_{PR}(\vect{x}_R))}{\min_{\vect{x}_R}\sum_{(P, R)\in\mc{E}_2}(\|\vect{m}_{P\rightarrow R}^{(t)}\|_2^2+\|{\mc{L}}(Y_{2,PR}^{(t)}(\vect{x}_R))\|_2^2)+\sum_{(P,R)\in\mc{E}_3}(\min_{\vect{x}_R} \mathrm{LP}_{PR}(\vect{x}_R))}\nonumber\\
&\triangleq \Lambda(\Phi',k_{lu}), \label{eq:upbnd1}
\end{align}
where we used the fact that $\vect{m}^{(t)}_{P\rightarrow R}$ and ${\mc{L}}(Y(S))$ sum to one. Now considering $H_2$, we can write the following lower bound
\begin{align}
H_2&\geq\frac{\nu}{2}\frac{\|\delta_{\mc{E}_{\sim 1}}\|^2_2}{\|\vect{m}^*_{\mc{E}_{\sim 1}}\|^2_2}\nonumber\\
&\geq \frac{\nu}{2}\frac{\| \vect{m}^{(t)}_{P\rightarrow R}-\vect{m}^*_{P\rightarrow R}\|_2^2}{\|\vect{m}^*_{P\rightarrow R}\|_2^2}. \label{eq:lowbnd1}
\end{align}
Taking expectation from both sides of bounds \eqref{eq:upbnd1} and \eqref{eq:lowbnd1} and putting them together we obtain
\begin{align}
\mathbb{E} \left[ {\delta}_{\mc{E}_{\sim 1}}^{(t+1)} \right] \leq \Lambda(\Phi',k_{lu}){(\alpha^{(t)})}^2+(1-\alpha^{(t)}) \mathbb{E} \left[ {\delta}_{\mc{E}_{\sim 1}}^{(t)} \right]. \label{eq:recursion_for_error}
\end{align}
Taking $\alpha^{(t)}=\frac{\alpha}{(\nu(t+2))}$ and unwrapping the recursion \eqref{eq:recursion_for_error}  we get 
\begin{align}\label{geq}
\mathbb{E} \left[ {\delta}_{\mc{E}_{\sim 1}}^{(t+1)}\right] \leq \frac{\Lambda(\Phi',k_{lu})\alpha^2}{\nu^2}\sum_{i=2}^{t+2} \left( \frac{1}{i^2}\prod_{n=i+1}^{t+2}(1-\frac{\alpha}{n}) \right)+\prod_{n=2}^{t+2} \left(1-\frac{\alpha}{n} \right)\mathbb{E} \left[ {\delta}_{\mc{E}_{\sim 1}}^0 \right]
\end{align}
adopting the convention that the inside product is equal to
one for $i = t + 2$. The following lemma, provides an upper bound on the product $\prod_{n=i+1}^{t+2}\big(1-\frac{\alpha}{n}\big)$.
\begin{lemma}[see \cite{noorshams_stochastic_2013}]
For $i\leq t+1$ we have 
\[
\prod_{n=i+1}^{t+2}\big(1-\frac{\alpha}{n}\big)\leq \big(\frac{i+1}{i+3}\big)^\alpha.
\]
\end{lemma}
Then plugging this lemma into \eqref{2node} and taking the same steps as in \cite[\S~IV-B-2]{noorshams_stochastic_2013}, we get the desired conclusion. 
%\end{proof}
%\textcolor{red}{UP TO HERE NOW.....}
%\begin{proof}
%\\
\vspace{8pt}
\item{\textit{High probability bounds on the actual error rate (proof of part (iii)):}}\\
To prove this part, again, we adapt the approach of \cite{noorshams_stochastic_2013} to SGBP algorithm. Reminding the definition of normalized error ${\delta}_{\set{E}\sim 1}^{(t)}$ from the previous part as well as \eqref{eq:divup}, we can write

\begin{align}
{\delta}_{\mc{E}_{\sim 1}}^{(t+1)}-{\delta}_{\mc{E}_{\sim 1}}^{(t)}&=\frac{[\langle \vect{m}_{\mc{E}_{\sim 1}}^{(t+1)}-\vect{m}_{\mc{E}_{\sim 1}}^{(t)},\vect{m}_{\mc{E}_{\sim 1}}^{(t+1)}+\vect{m}_{\mc{E}_{\sim 1}}^{(t)}-2\vect{m}_{\mc{E}_{\sim 1}}^*\rangle]}{\|\vect{m}_{\mc{E}_{\sim 1}}^*\|^2_2}\nonumber \\
&={(\alpha^{(t)})}^2\frac{[\|U(\vect{m}_{\mc{E}_{\sim 1}}^{(t)},J_{\mc{E}_{\sim 1}}^{(t+1)})\|^2_2]}{\|\vect{m}_{\sim 1}^*\|^2_2}\nonumber\\
&\quad -2{(\alpha^{(t)})}\frac{[\langle u(\vect{m}_{\mc{E}_{\sim 1}}^{(t)}),\vect{m}_{\mc{E}_{\sim 1}}^{(t)}-\vect{m}_{\mc{E}_{\sim 1}}^*\rangle]}{\|\vect{m}_{\mc{E}_{\sim 1}}^*\|^2_2}+2\alpha^{(t)} \langle Z^{(t+1)},\delta_{{\mc{E}_{\sim 1}}}^{(t)}\rangle\label{eq:genbnd}
\end{align}
where $Z^{(t+1)}$ is defined as following 
\[
Z^{(t+1)}\triangleq -\frac{ \left[ \vect{U}(\vect{m}_{\sim 1}^{(t)},J^{(t+1)})-\vect{u}(\vect{m}_{\set{E}\sim 1}^{(t)})\right]}{\|\vect{m}_{\set{E}\sim 1}^*\|^2_2}.
\]
Next, plugging \eqref{eq:upbnd1} and \eqref{eq:lowbnd1} into \eqref{eq:genbnd}, we obtain
\begin{align*}
\delta_{\mc{E}_{\sim 1}}^{(t+1)}\leq \Lambda(\Phi',k_{lu})(\alpha^{(t)})^2+2\alpha^{(t)} \langle Z^{(n+1)},\delta_{{\mc{E}_{\sim 1}}}^n\rangle+(1-\nu\alpha^{(t)})\delta_{\mc{E}_{\sim 1}}^{(t)}.
\end{align*}
Setting $\alpha^{(t)}=\frac{1}{\nu(t+1)}$, and doing some algebra with the recursion, yields to
\begin{align}
\delta_{\mc{E}_{\sim 1}}^{(t+1)} &\leq \frac{\Lambda(\Phi',k_{lu})}{\nu^2(t+1)}\sum_{n=1}^{t+1}\frac{1}{n}+\sum_{n=1}^{t+1}\frac{2\langle Z^{(t+1)},\delta_{{\mc{E}_{\sim 1}}}^n\rangle}{\nu(t+1)}\nonumber\\
&\leq\frac{\Lambda(\Phi',k_{lu})}{\nu^2(t+1)}\frac{1+\log(t+1)}{t+1}+\sum_{n=1}^{t+1}\frac{2\langle Z^{(n+1)},\delta_{{\mc{E}_{\sim 1}}}^n\rangle}{\nu(t+1)}.\label{eq:upbnd2}
\end{align}
Notice that $\{Z^{n}\}_{n=1}^{\infty}$ can be interpreted as a martingale difference sequence regarding the filtration $\mathcal{F}^n=\sigma(\vect{m}_1^0,\vect{m}_1^1,\ldots,\vect{m}_1^n)$. Therefore, $\mathbb{E}[Z^{(n+1)}|\mc{F}^{(t)}]=\vect{0}$ and correspond to that $\mathbb{E}[\langle Z^{(n+1)},\delta^n_{{\mc{E}_{\sim 1}}}\rangle]=0$ for $n=0,1,2,\ldots$. To bound the left hand side of \eqref{eq:upbnd2}, we need to bound second term's variance in \eqref{eq:upbnd2} as following
\begin{align*}
\text{var} \left( \frac{1}{n+1}\sum_{n=1}^{(t+1)}{\langle Z^{n+1},\delta_{{\mc{E}_{\sim 1}}}^n\rangle} \right) &=\frac{1}{(n+1)^2}\mathbb{E} \left[(\sum_{n=0}^n\langle Z^{(n+1)},\delta^n_{{\mc{E}_{\sim 1}}}\rangle)^2\right]\\
&= \frac{1}{(n+1)^2}\sum_{n=0}^n\mathbb{E} \left[ \langle Z^{(n+1)},\delta^n_{{\mc{E}_{\sim 1}}}\rangle^2 \right]\\
&\quad + \frac{2}{(n+1)^2}\sum_{0\leq n_1<n_2\leq t}\mathbb{E} \left[ \langle Z^{n_1+1},\delta_{{\mc{E}_{\sim 1}}}^{n_1}\rangle\langle Z^{n_2+1},\delta_{{\mc{E}_{\sim 1}}}^{n_2}\rangle \right].
\end{align*}
Because of the fact that 
\begin{align}
\mathbb{E} \left[ \langle Z^{(n_1+1)},\delta_{{\mc{E}_{\sim 1}}}^{n_1}\rangle\langle Z^{(n_2+1)},\delta_{{\mc{E}_{\sim 1}}}^{n_2}\rangle \right] =& \mathbb{E} \left[ \mathbb{E}\big[\langle Z^{(n_1+1)},\delta_{{\mc{E}_{\sim 1}}}^{n_1}\rangle\langle Z^{(n_2+1)},\delta_{{\mc{E}_{\sim 1}}}^{n_2}\rangle|\mc{F}^{n_1}\big] \right]\nonumber\\
=&\mathbb{E} \left[ \langle Z^{(n_1+1)},\delta_{{\mc{E}_{\sim 1}}}^{n_1}\rangle\mathbb{E}\big[\langle Z^{(n_2+1)},\delta_{{\mc{E}_{\sim 1}}}^{n_2}\rangle|\mc{F}^{n_1}\big] \right] \label{eq:crossproduct-term}\\
=&0\nonumber
\end{align} 
$\forall n_1<n_2$, the second term in \eqref{eq:crossproduct-term} becomes zero. This means that the martingale different sequence is bounded.

Using the concentration inequality provided in \cite{chung2006concentration}, we can bound the second term in \eqref{eq:upbnd2}. To this end, first, we need to upper bound  $\|Z^{(n)}\|_2$ by Cauchy-Schwartz inequality which leads to the following upper bound
\begin{align*}
\|Z^{(n+1)}\|_2 &= \frac{[\| U(\vect{m}_{\set{E}\sim 1}^{(n)},J^{n+1})- \vect{u}(\vect{m}_{\set{E}\sim 1}^{(n)})]\|_2}{\|\vect{m}_{\set{E}\sim 1}^*\|_2}\\
&\leq \left( \frac{\|U(\vect{m}_{\set{E}\sim 1}^{(n)}, \vect{J}^{(n+1)})\|_2}{\|\vect{m}_{\set{E}\sim 1}^*\|_2}+\frac{\| \vect{u}(\vect{m}_{\set{E}\sim 1}^{(n)})\|_2}{\| \vect{m}_{\set{E}\sim 1}^* \|_2} \right)\\
&\stackrel{(a)}{\leq}2\sqrt{\Lambda(\Phi',k_{lu})}
\end{align*}
where in (a) we used the convexity of the second norm ($\|.\|_2$) to apply the Jensen's inequality which implies that $\frac{\| \vect{u}(\vect{m}_{\set{E}\sim 1}^{(n)})\|_2}{\| \vect{m}_{\set{E}\sim1}^*\|_2}\leq \sqrt{\Lambda(\Phi',k_{lu})}$.  
By using \eqref{eq:upbnd1}, we can conclude that $|Z^{(n)}|\leq 2\Lambda(\Phi',k_{lu})$, for $\forall n=0,1,\ldots$. 
Now, turning to bound the second term, again we use the Cauchy-Schwartz inequality as following
\begin{align*}
\mathbb{E} \left[ \langle Z^{(n+1)},\delta^n_{\mc{E}_{\sim 1}}\rangle^2 \right] &\leq \mathbb{E} \left[ \| Z^{(n+1)}\|_2^2 \cdot \| \delta^n_{\mc{E}_{\sim 1}}\|_2^2 \right]\\
&\leq 4\Lambda(\Phi',k_{lu}) \| \delta^n_{\mc{E}_{\sim 1}}\|_2^2.
\end{align*} 
Next, taking the expectation over the both sides of the inequality \eqref{eq:upbnd2}, we obtain $\mathbb{E}\big[\|\delta^n_{{\mc{E}_{\sim_{1}}}}\|_2^2\big]\leq \frac{\Lambda(\Phi',k_{lu})}{\nu^2}\frac{(1+\log n)}{n}$. Thus, we have
\[
\mathbb{E}\big[\langle Z^{(n+1)},\delta^n_{\mc{E}_{\sim 1}}\rangle^2\big]\leq\frac{4\Lambda^2(\Phi',k_{lu})}{\nu^2}\frac{(1+\log n)}{n}
\]
for all $n\geq 1$. Furthermore, noting that the initial term $\mathbb{E}\big[\langle Z^{(1)},\delta^0\rangle^2\big]\leq 4\Lambda(\Phi',k_{lu})\mathbb{E}\big[\|\delta^0\|_2^2\big]$ is upper bounded by $4\Lambda(\Phi',k_{lu})$ using the fact that \begin{align*}
\frac{\mathbb{E}\| m_{\mc{E}_{\sim 1}}^{(0)} - m_{\mc{E}_{\sim 1}}^*\|}{\| m_{\mc{E}_{\sim 1}}^*\|} &\leq 2\frac{\mathbb{E}\| m_{\mc{E}_{\sim 1}}^0\|}{\| m_{\mc{E}_{\sim 1}}^*\|}\\
&\leq 2\sqrt{\frac{\Lambda(\Phi',k_{lu})}{4}}.
\end{align*}
Now, we have everything to bound the variance as following
\begin{align*}
\text{var} \left( \frac{1}{t+1}\sum_{n=0}^{(t)}\langle Z^{n+1},\delta^{n}\rangle\right)  &\leq \frac{4\Lambda^2(\Phi',k_{lu})}{\nu^2}\sum_{n=1}^{t}\frac{(1+\log n)}{n}+\frac{4\Lambda^2(\Phi',k_{lu})}{(t+1)^2}\\
&\stackrel{(a)}{\leq} \frac{4\Lambda^2(\Phi',k_{lu})}{\nu^2}\frac{(1+\log (t+1))^2+4}{(t+1)^2}
\end{align*}
where (a) follows because of the elementary inequality
\[
\sum_{n=1}^{(t)}\frac{(1+\log n)}{n}\leq (1+\log t)^2
\]
and also $\mu\leq 2$. 
As a result, everything is ready to use the Chebyshev's inequality \cite{chung2006concentration}  for upper bounding. Thus we have
\[
\Pr\left( |\sum_{n=1}^{(t+1)}\frac{2\langle Z^{n},\delta^n\rangle}{\nu(t+1)}|\geq \tau \right) \leq \frac{16\Lambda^2(\Phi',k_{lu})}{\nu^4\tau^2}\frac{(1+\log (t+1))^2+4}{(t+1)^2}
\]
for any $\tau>0$, so we can choose 
\[
\tau=\left( \frac{4\Lambda(\Phi',k_{lu})}{\sqrt{\epsilon}\nu^2} \right) \frac{\sqrt{(1+\log(t+1))^2+4}}{(t+1)^2},
\]
for a fixed $0<\epsilon<1$. Finally, putting all together in \eqref{eq:upbnd2} along some simplifications ends up with \begin{align*}
\delta_{\set{E}_{\sim 1}}^{(t+1)}\leq \frac{\Lambda(\Phi',k_{lu})}{\nu^2}\frac{1+\log(t+1)}{t+1}+\frac{4\Lambda(\Phi',k_{lu})}{\nu^2\sqrt{\epsilon}}\frac{\sqrt{(1+\log(t+1))^2+4}}{t+1}
\end{align*}
with the least probability  $1-\epsilon$.
\end{itemize}
\end{proof}
%\fi

% Bibliography-------------------------------------------------------
\bibliographystyle{IEEEtran}
%\bibliographystyle{alpha}
%\bibliography{lth,lthpub}
%\input{bib.tex}
%\bibliographystyle{IEEEtranS}
\bibliography{MyLibrary}

%\begin{thebibliography}{1}
%\bibitem{Yedida}
%J. S. Yedidia, W. T. Freeman, and Y. Weiss, ``Constructing free energy approximations and generalized belief propagation Algorithms,'' IEEE Trans. Inf. Theory, vol. 51, no. 7, pp. 2282–2312, Jul. 2005.
%  
%\bibitem{Noorshams}
%  Noorshams, Nima, and Martin J. Wainwright. ``Stochastic belief propagation: A low-complexity alternative to the sum-product Algorithm." Information Theory, IEEE Transactions on 59.4 (2013): 1981-2000.
%  \bibitem{Pakzad}
%  Pakzad, Payam, and Venkat Anantharam. "Estimation and marginalization using the kikuchi approximation methods." Neural Computation 17.8 (2005): 1836-1873.
%  \end{thebibliography}
\end{document}